\documentclass[11pt]{article}
\usepackage[utf8]{inputenc}
\def\withnotes{1}
\def\withcolors{1}
\usepackage{tex_macros}
\usepackage{graphicx}
\usepackage{subfigure}
\usepackage{bm}
\usepackage{fullpage}
\usepackage{framed}
\usepackage{physics}
\usepackage{nicefrac}
\usepackage{xcolor,cancel}
\usepackage[round]{natbib}
\usepackage{wrapfig}

\newtheorem*{lemma*}{Lemma}
\newenvironment{proof_outline}%
  {%
   \par\noindent{\bfseries\upshape Proof Outline\ }%
  }%
  {\qed}

\newtheorem{assumption}[theorem]{Assumption}

\title{Asymptotic Learning Curves for Diffusion Models with\\ Random Features Score and Manifold Data}

\author{Anand Jerry George
        \and Nicolas Macris}

\date{École Polytechnique Fédérale de Lausanne (EPFL),\\
Lab for Statistical Mechanics of Inference in Large Systems (SMILS),\\ CH-1015 Lausanne,\\ Switzerland}
\begin{document}
\maketitle

\begin{abstract}
We study the theoretical behavior of denoising score matching--the learning task associated to diffusion models--when the data distribution is supported on a low-dimensional manifold and the score is parameterized using a random feature neural network. We derive asymptotically exact expressions for the test, train, and score errors in the high-dimensional limit. Our analysis reveals that, for linear manifolds the sample complexity required to learn the score function scales linearly with the intrinsic dimension of the manifold, rather than with the ambient dimension. Perhaps surprisingly, within our model, the benefits of low-dimensional structure starts to diminish once we have a non-linear manifold. These results indicate that diffusion models can benefit from structured data; however, the dependence on the specific type of structure is subtle and intricate.
\end{abstract}

\section{Introduction}
\subsection{Generative modelling}
In generative modelling, we are concerned with the following problem: Given a set $\cS$ of i.i.d. samples $\{x_i\}_{i=1}^n$ from an unknown probability distribution $P_0$, we want to generate a new sample from $\pi$ independent of $\cS$. Diffusion models have recently emerged as a powerful class of generative models, achieving state-of-the-art performance in high-dimensional data generation tasks such as image, audio, and molecular synthesis. These models are based on learning the \textit{score function}—the gradient of the log-density—of a sequence of progressively noised versions of the data distribution, typically via denoising score matching (DSM). Once the score function is learned, sampling can be performed by simulating a reverse-time stochastic differential equation or its discretizations, providing a flexible framework for generative modeling.

Despite their empirical success, the theoretical understanding of diffusion models remains limited, particularly in regimes that reflect the structure of real-world data. A commonly held hypothesis in machine learning is that high-dimensional data concentrate near low-dimensional structures, often idealized as smooth manifolds embedded in ambient Euclidean space. This manifold property is conjectured to be the primary reason behind the tractability of many high-dimensional problems that are otherwise plagued by the \textit{curse of dimensionality}. However, its implications for diffusion-based generative models and score matching objectives are not yet fully understood. A theoretical study of diffusion models for manifold data first appeared in \cite{pidstrigach_score-based_2022}. They obtained conditions on score for diffusion models to sample from the data manifold. Convergence of the backward dynamics for manifold data was studied in \cite{bortoli_convergence_2022}. In particular, they assumed that the score function is learned to a certain accuracy apriori. More recently, \cite{azangulov_convergence_2025} analyzed diffusion models end-to-end for manifold data. 

In this work, we develop a theoretical analysis of denoising score matching when the data distribution is supported on a smooth manifold and a random feature neural network (RFNN) is used to learn the score function. The key difference of our contributions from prior works is that we rely on an asymptotically precise characterization instead of non-asymptotic bounds.

\subsection{Context of This Work}
A key question in the theory of diffusion models is the sample complexity of devising an approximate score function. When using an \textit{empirical optimal score} (see eqn~\eqref{eqn:emp_opt_score}) function, \cite{biroli_dynamical_2024} showed that the required number of samples grows exponentially with the ambient dimension $d$. Subsequent works \cite{achilli_memorization_2025,achilli_losing_2024, george_analysis_2025} demonstrated that when the data distribution is supported on a low-dimensional manifold, the exponential dependence shifts from the ambient dimension $d$ to the intrinsic dimension $D$. These results provided theoretical evidence that geometric structure alone can reduce the effective sample complexity significantly. Further, \cite{george_denoising_2025} showed that when the score is parameterized using a random feature neural network, the sample complexity scales linearly with the ambient dimension. Clearly, the limited approximation capacity of score parameterizations and the data structure plays key role in the practical success of diffusion models. \textit{The present work takes a step further: we consider the situation when both sources of structure are taken into account—namely, a lower-dimensional manifold model for the data and a random feature parameterization for the score.}

\subsection{Our Contributions}
We use a random feature neural network (RFNN) (see Sec.~\ref{sec:RFNN}) to parameterize the score function, and hidden manifold model (HMM) (see Sec. \ref{sec:HMM}) for data. We work in a regime where the ambient data dimension $d$, intrinsic dimension $D$, number of samples $n$, and number of neurons in RFNN $p$ go to infinity, while the ratios $\psi_D = \frac{D}{d}, \psi_n = \frac{n}{d}, \psi_p = \frac{p}{d}$ stay fixed. In this regime,
\begin{enumerate}
    \item We derive a precise, asymptotic characterization of test and train errors (eqns~\ref{eqn:test_err},\ref{eqn:train_err}) for the minimizer of denoising score matching (DSM) loss \eqref{eqn:dsm_rfnn}.
    \item In addition, we also derive the score error (eqn~\ref{eqn:score_err}) using the test error and the MMSE estimates for Generalized Linear Models obtained in \cite{barbier_optimal_2019}. 
    \item We demonstrate that for `sufficiently linear' manifolds, the number of samples required to learn the score function to a certain accuracy depends linearly on the intrinsic dimension $D$.
    \item Our mathematical analysis also yields a refined characterization of the Gaussian equivalence principle, which may be of independent interest.
\end{enumerate}

\subsection{Related Works}
Diffusion models \cite{sohl-dickstein_deep_2015,song_generative_2019,ho_denoising_2020,song_score-based_2020} are a class of generative models that leverage the non-equilibrium dynamics of diffusion processes to model complex data distributions. Since their introduction, a number of architectural and algorithmic advances \cite{dhariwal_diffusion_2021,rombach_high-resolution_2022,ho_classifier-free_2021,nichol_glide_2022} have established diffusion models as state-of-the-art methods for high-fidelity image generation.

Alongside these empirical successes, a growing body of work has investigated the theoretical foundations of diffusion models. Several studies analyze the accuracy of the sampling procedure by bounding the distance between the generated and target distributions \cite{chen_sampling_2022,benton_nearly_2023,chen_improved_2023,bortoli_convergence_2022}. These results typically assume access to a score function that is learned \emph{a priori} with a prescribed accuracy. Among them, \cite{bortoli_convergence_2022} provides convergence guarantees in the setting where the data distribution is supported on a manifold.

Complementary to sampling analyses, recent works focus on understanding the learning of the score function itself \cite{cui_analysis_2023,shah_learning_2023,han_neural_2023,zeno_when_2025}. End-to-end theoretical studies of diffusion models \cite{kadkhodaie_generalization_2023,chen_score_2023,li_generalization_2023,wang_evaluating_2024,cui_precise_2025} further shed light on generalization and memorization phenomena. In particular, \cite{li_generalization_2023} and \cite{saha_generalization_2025} consider score functions parameterized by random feature neural networks and derive bounds on the KL divergence between the learned and target distributions. The geometric properties of manifolds implicitly learned by diffusion models are analyzed in \cite{pidstrigach_score-based_2022}.

From a statistical physics perspective, memorization in high-dimensional generative models has been studied using the empirical optimal score function in \cite{biroli_dynamical_2024,raya_spontaneous_2023,ambrogioni_search_2024,achilli_losing_2024,achilli_memorization_2025}. A geometric viewpoint on memorization is proposed in \cite{ross_geometric_2024}, while \cite{bonnaire_why_2025} investigates how early stopping can mitigate memorization effects.

\section{Preliminaries}
We briefly discuss some basics of diffusion models, RFNN, and hidden manifold model. 
\paragraph{Diffusion Models}
Consider a set of $n$ i.i.d. samples $\cS = \{x_i\}_{i=1}^n$ from an unknown distribution $P_0$ on $\R^d$. Generative modeling aims to leverage the information in this set to draw new samples from $P_0$. Diffusion models address the problem by time reversing a diffusion process that transports $P_0$ to a known distribution such as a Gaussian. 
In this work, we let the forward process to be a  Ornstein-Uhlenbeck (OU) process. The stochastic differential equation (SDE) corresponding to an OU process is
\begin{equation}\label{eqn:forward_ou}
    dX_t = -X_t dt + \sqrt{2} dB_t\;,\quad X_0\sim P_0\;.
\end{equation}
Here, $B_t$ is a standard $d$-dimensional Brownian motion. The distribution of $X_t$ given $X_0$ can be computed in closed form and is given by $\cN{e^{-t}X_0,(1-e^{-2t})I_d}$. As $t\to\infty$, the distribution of $X_t$ tends to be the $d$-dimensional Gaussian distribution with covariance $I_d$, regardless of $X_0$. 
Let $P_t$ denote the probability distribution of $X_t$:
\begin{equation}\label{eqn:marginal_prob_forward}
    P_t(x) = {(2\pi h(t))^{-d/2}}\int_{\R^d}  d P_0(x_0) \; e^{-\frac{\norm{x-a_tx_0}^2}{2h(t)}}\;,
\end{equation} 
where $a_t=e^{-t}, h_t = 1-e^{-2t}$.
Then, for a fixed $T>0$ and $Y_T\sim P_T$, we define the \textit{time reversal} of the forward process (\ref{eqn:forward_ou}) as 
\begin{equation}\label{eqn:backward_sde_brown}
    -d Y_t = (Y_t+2\nabla\log P_{t}(Y_t)) \;  d t+ \sqrt{2}\; d \tilde{B}_t\quad \;,
\end{equation}
where the SDE runs backward in time starting from $Y_T$, and $\tilde{B}_t$ is a different instance of standard Brownian motion.
The term \emph{time reversal} \citet{anderson_reverse-time_1982} here means that the distributions of $Y_t$ and $X_t$ are identical for every $t$.
If we initiate the backward process with $Y_T\sim P_T$, the distribution of $Y_0$ will be $P_0$. However, since $P_T$ is unknown due to the lack of knowledge of $P_0$, we instead start the reverse process with $Y_T\sim\cN{0,I_d}$ which is a reasonable approximation for large $T$.

The main ingredient required to implement the backward process is $\nabla\log P_t$, known as the \textit{score} function of $P_t$. We call this the \emph{exact score} in order to distinguish it from the learned score used in practice. The learning task is to obtain a reasonable estimate of the exact score using the dataset $\cS$. A possible approach is to minimize the following score matching objective: 
 $\cL_{\text{SM}}(s) = \int_0^T  \dd t \; \shortexpect_{x_t\sim P_t}{\norm{s(t,x_t)-\nabla\log P_t(x_t)}^2}$.
However, the $\cL_{\text{SM}}$ loss function is not practical, as $\nabla\log P_t$ is unknown. Nevertheless, it is possible to construct an equivalent objective, the denoising score matching (DSM) loss \cite{vincent_connection_2011}: $\cL_{\text{DSM}}(s) = \int_0^T  \dd t \; w(t)\shortexpect\norm{s(t,x_t)-\nabla\log P_t(x_t|x_0)}^2 ,$
where $w$ is a weighting function and the expectation is with respect to $x_0$ and $x_t$.
Following \cite{song_score-based_2020}, we choose $w(t) = (\shortexpect_{x_0,x_t}{\norm{\nabla\log P_t(x_t|x_0)}^2})^{-1}$.
For OU process, we can compute $\nabla\log P_t(x_t|x_0)$ in closed form. We can write $x_t\sim P_t$ as  $x_t = a_t x_0+\sqrt{h_t}z$, where $x_0\sim P_0,\; z\sim\cN{0,I_d}$ are independent rvs and $a_t=e^{-t}$,\; $h_t=1-e^{-2t}$. Consequently, $\nabla\log P_t(x_t|x_0) = -\frac{(x_t-a_tx_0)}{h_t}=-\frac{z}{\sqrt{h_t}}$. The weight function is given by $w(t)=\frac{h_t}{d}$. Substituting these, we can write $\cL_{\text{DSM}}(s) = \int_0^T  \dd t  \frac{1}{d}\shortexpect\norm{\sqrt{h_t}s(t,a_tx_0+\sqrt{h_t}z)+z}^2,$
where the expectation is with respect to $x_0$ and $z$. Since $P_0$ is unknown and only samples from it are available, we use an empirical estimate for the expectation with respect to $x_0$. Finally, we get
\begin{align}\label{eqn:dsm_minf}
    \cL(s) &= \int_0^T \dd t \; \frac{1}{dn}\sum_{i=1}^{n}{\shortexpect_{z}{\norm{\sqrt{h_t}s(t, a_tx_i+\sqrt{h_t}z)+z}^2}}, 
\end{align}
which should, in theory, be minimized.

\paragraph{Score model: Random Features Neural Network}\label{sec:RFNN}
In practice, the score function $s$ is typically chosen from a parametric class of functions, and the DSM objective (\ref{eqn:dsm_minf}) is minimized within this class, with an appropriate regularization. In this work, we represent the score function using a \textit{random features} neural network (RFNN) \cite{rahimi_random_2007}. A RFNN is a two-layer neural network in which the first layer weights are randomly chosen and fixed, while the second layer weights are learned during training. It is a function from $\R^d$ to $\R^d$ of the form
$s_A(x|W) = \rat{A}{p}\act{\rat{W}{d}x},$
where $W\in\R^{p\times d}$ is a random matrix with its elements chosen i.i.d. from $\cN{0,1}$, $\varrho$ is an activation function acting element-wise and $A\in\R^{d\times p}$ are the second layer weights that need to be learned. The RFNN
(usually considered for scalar output) is a simple neural network  amenable to theoretical analysis and is able to capture interesting characteristics observed in more complicated neural network models, such as the double descent curve related to overparametrized regimes~\cite{mei_generalization_2022,bodin_model_2021,bodin_gradient_2022}.

\paragraph{Data Model: Hidden Manifold Model}\label{sec:HMM}
We consider a hidden manifold model \cite{goldt_modeling_2020,goldt_gaussian_2022} for data, which we define as follows: let $\{\xi_i\}_{i=1}^n$ be i.i.d. with $\xi_i\sim\cN{0,I_D}$, where $D\le d$. Let $M\in\R^{d\times D}$ be a random matrix with i.i.d. $\cN{0,1}$ entries. Let $\sigma$ be a non-linearity that acts entry-wise. Then, input data $\cS = \{x_i\}_{i=1}^n$ is defined using $x_i=\man{\rat{M}{D}\xi_i}$. In this model the dataset $\mathcal{S}$ lies in a $D$-dimensional manifold embedded in an ambient space $\mathbb{R}^d$. Intuitively, this manifold corresponds to taking a $D$-dimensional hyperplane defined by the matrix $M$, and applying a deformation, thanks to a smooth function $\sigma$.
\section{Main Results}
We characterize the asymptotic test, train, and score errors (see Sec.~\ref{sec:eval_metrics} for their definitions) for the minimizer of denoising score matching loss when the score function is parameterized using a RFNN and the data comes from HMM. Our results hold in the high-dimensional regime where $d,D,n,p\to\infty$, while the ratios $\psi_D=\frac{D}{d},\psi_n=\frac{n}{d},$ and $\psi_p=\frac{p}{d}$ are fixed.

We further assume that, at each time 
$t$, an independent RFNN is used to learn the score function corresponding to that time. While this assumption simplifies the analysis relative to practical implementations, it is often used in prior theoretical works (see for e.g., \citet{cui_analysis_2023,george_denoising_2025,bonnaire_why_2025}). Under this assumption, minimizing the DSM loss \eqref{eqn:dsm_minf} is equivalent to minimizing its integrand separately at each time 
$t$. Accordingly, we restrict our attention to the minimization problem at a fixed time $t$. After introducing a regularizaiton parameter $\lambda>0$, the loss function \eqref{eqn:dsm_minf} for RFNN score at a fixed time $t$ becomes
\begin{equation}\label{eqn:dsm_rfnn}
    \cL_t(A_t) =\frac{1}{dn}\sum_{i=1}^{n}{\shortexpect_{z}{\norm{\sqrt{h_t}s_{A_t}(a_tx_i+\sqrt{h_t}z|W_t)+z}^2}}+\frac{h_t\lambda}{dp}\norm{A_t}_F^2.
\end{equation}
Since \eqref{eqn:dsm_rfnn} is a squared loss, its minimizer $\hat{A}_t$ can be written in closed form and is given in Appendix~\ref{apndx_sec:optimal_rfnn_score}. Next, we define the learning errors for which we derive asymptotic expressions.
\subsection{Evaluation Metrics: Test, Train, and Score Errors}\label{sec:eval_metrics}
We evaluate the score learning process through {\it test}, {\it train}, and {\it score} errors. For the minimizer of \eqref{eqn:dsm_rfnn} $\hat{A}_t$ the expressions for these errors are given as follows:
\begin{align}
    \cE_{\text{test}}(\hat{A}_t) &= \frac{1}{d}\shortexpect_{x\sim P_0,z\sim\cN{0,I}}{\norm{\sqrt{h_t}s_{\hat{A}_t}(a_tx+\sqrt{h_t}z|W_t)+z}^2} \label{eqn:test_err}\;,\\
    \cE_{\text{train}}(\hat{A}_t) &= \frac{1}{dn}\sum_{i=1}^{n}\shortexpect_z{\norm{\sqrt{h_t}s_{\hat{A}_t}(a_tx_i+\sqrt{h_t}z|W_t)+z}^2} \label{eqn:train_err}\;,\\
    \cE_{\text{score}}(\hat{A}_t) &= \frac{1}{d}\shortexpect_{x\sim P_t}{\norm{s_{\hat{A}_t}(x|W_t)-\nabla\log P_t(x)}^2} \label{eqn:score_err}\;.
\end{align}
Note that $\cE_{\text{train}},\cE_{\text{test}}$ and $\cE_{\text{score}}$ are random due to $W_t$ and $M$. In the next section, we characterize the high-dimensional limits of $\bE{W_t,M}{\cE_{\text{train}}}$ and $\bE{W_t,M}{\cE_{\text{test}}}$ through random matrix techniques. We remark that the  $\cE_{\text{test}}, \cE_{\text{train}}$, and $\cE_{\text{score}}$ are expected to concentrate around their expectations. However, proving this is beyond the scope of the current work.

The score error $\cE_{\text{score}}$ is not directly computable since there is no closed-form expression for $\nabla\log P_t$. However, we make interesting connections to the free energy of Generalized Linear Models, which enables us to derive the asymptotic score error as well. 
\subsection{Test and Train Errors for the Optimal RFNN Score}\label{sec:test_train_RFNN}
In Theorem~\ref{thm:test_train}, we present the asymptotic expressions for test and train errors. We make the following assumption on the activation function $\act{\cdot}$ and the manifold folding function $\man{\cdot}$.
\begin{assumption}\label{assumption:act_man}
    We assume that $\varrho$ and $\sigma$ are Lipschitz functions. In addition, they satisfy the following conditions: with $g\sim\cN{0,1}$, $\shortexpect{\act{g}} = \shortexpect{\man{g}} = 0$, $\shortexpect{\act{g}^2} = \shortexpect{\man{g}^2} = 1$, $\bE{}{g\act{g}} = \mu_1, \bE{}{g\man{g}} = \nu_1$.
\end{assumption}
\begin{definition}
    The function $c$ is defined as $c(\gamma) = \bE{u,v\sim P^\gamma}{\varrho(u)\varrho(v)}$, with $P^\gamma$ denoting the bivariate standard Gaussian distribution with correlation coefficient $\gamma$ (see Eq.~\eqref{eqn:bivariate_gaussian_pdf} in the Appendix~\ref{appnxd_sec:proof_test_train}). 
\end{definition}
In Theorem~\ref{thm:test_train} we need a sufficiently accurate estimate for the second moment of the feature vectors. 
Towards this, we make the following claim on the second moment of feature vectors. Proof of a restricted version of the claim can be found in Appendix~\ref{apndx_sec:proof_lemma_gep}.
\begin{claim}\label{lemma:gep_sec_moment}
    Let $f:\R\to\R$ be any smooth function such that $\bE{g\sim\cN{0,1}}{f(g)} = 0$. Let $\phi_i = \frac{w_i^T}{\sqrt{d}}\man{\rat{M}{D}\xi}$ and $\phi_i' = \frac{w_i^T}{\sqrt{d}}\left(\nu_1\rat{M}{D}\xi'+\sqrt{1-\nu_1^2}z\right)$ for $i=1,2$, where $w_1,w_2,z\stackrel{i.i.d.}{\sim}\cN{0,I_d}$ and $\xi,\xi'\stackrel{i.i.d.}{\sim}\cN{0,I_D}$ are independent rvs. Then, for all sufficiently small $\epsilon>0$ 
    \begin{align}
        \big|\bE{\xi}{f(\phi_1)f(\phi_2)}-\bE{\xi',z}{f(\phi_1')f(\phi_2')}\big| &= O(1/d^{1-\epsilon}), \quad \text{w.h.p.}\label{eqn:gep_sec_moment}
    \end{align}
    where w.h.p means with probability tending to $1$ faster than any inverse power in $d$. 
\end{claim}
\paragraph{Proof outline in a restricted setting}
We prove the claim in the case when $f$ is a polynomial and $\sigma$ is an activation function with finite number of Hermite coefficients. In the proof, we upper bound the derivative of an interpolating quantity $S(t) := \bE{\xi,\xi',z}{f(\phi_1^t)f(\phi_2^t)}$, where $\phi_i^t = \sqrt{t}\phi_i+\sqrt{1-t}\phi'_i$. We accomplish this by controlling joint cumulants of $(\phi_1,\phi_2)$. The probability is controlled thanks to hypercontractive bounds and is of order $1-O(e^{-a_{f, \sigma}d^\epsilon})$ for some $a_{f,\sigma}>0$ positive constant depending on $f$ and $\sigma$.

\begin{theorem}\label{thm:test_train}
Let $M\in\R^{d\times D}$ be random matrix with i.i.d. $\cN{0,1}$ entries, and let $\{\xi_i\}_{i=1}^n$ be i.i.d $\cN{0,I_D}$. The dataset $\cS = \{x_i\}_{i=1}^n$ is obtained using $x_i = \man{\rat{M}{D}\xi_i}$ and we denote the distribution of $x_is$ by $P_0$. Let $\varrho$ and $\sigma$ satisfy Assumption~\ref{assumption:act_man}, and define $s^2 = 1-c(a_t^2)-h_t\mu_1^2$, and $v^2 = 1-\mu_1^2$. Let the ratios between dimensions be fixed and given by $\psi_D = \frac{D}{d}$, $\psi_n = \frac{n}{d}$, and $\psi_p = \frac{p}{d}$. 
Suppose that, for $(q,z)$ in a neighborhood of $(0,-\lambda)$, there exists a differentiable solution $\zeta_1,\zeta_2,\zeta_3,\zeta_4$ of the self-consistent system
\begin{equation*}
\begin{cases}
    \quad\psi_n a_t^4\left(\frac{\mu_1^2}{\chi} + q\right)^2 \zeta_1+\frac{\psi_D}{\nu_1^2} h_t(\mu_1^2+q)\zeta_4+\frac{\psi_D}{\psi_p\nu_1^2}\left(s^2 + \frac{v^2}{\chi} - z\right)\zeta_4 (1+\zeta_3)-a_t^2\left(\frac{\mu_1^2}{\chi} + q\right) = 0\\
    \quad\frac{1+\zeta_3}{\psi_p}-
\frac{(s^2 + \frac{v^2}{\chi} - z)}{\nu_1^2\, a_t^2 \left(\frac{\mu_1^2}{\chi} + q\right)}\frac{(1+\zeta_3)^2 \, \psi_D \, \zeta_4}{\psi_p^{\,2}}-\zeta_3 = 0,\\
\quad a_t^2\left(\frac{\mu_1^2}{\chi} + q\right)\psi_n \zeta_1+\frac{(h_t\mu_1^2+q)\,\psi_D \zeta_4}{\nu_1^2 a_t^2\left(\frac{\mu_1^2}{\chi}+q\right)}+\left(s^2 + \frac{v^2}{\chi} - z\right)\psi_n \zeta_2-\psi_p = 0,\\
\quad\frac{\psi_D \zeta_4}{1 + \zeta_4} + \psi_D \zeta_4 \left( \frac{1 - \nu_1^2}{\nu_1^2} + \frac{h (\mu_1^2 + q)}{a_t^2 \left( \frac{\mu_1^2}{\chi} + q \right) \nu_1^2} + \frac{\left( s^2 + \frac{v^2}{\chi} - z \right)}{a_t^2 \left( \frac{\mu_1^2}{\chi} + q \right)} \frac{(1 + \zeta_3)}{\psi_p \nu_1^2} \right) - 1 = 0,
\end{cases}
\end{equation*}
where $\chi(\zeta_1,\zeta_2) = 1+a_t^2\mu_1^2\zeta_1+v^2\zeta_2$.
Then, for the minimizer of \eqref{eqn:dsm_rfnn} $\hat{A}_t$, we have
\begin{align*}
    \lim_{d\to\infty}\bE{}{\cE_{\text{test}}(\hat{A}_t)}
    &= 1-2h_t\mu_1^2\cK(0,-\lambda)-h_t\mu_1^4\frac{\partial \cK}{\partial q}(0,-\lambda)+h_t\mu_1^2(1-\mu_1^2)\frac{\partial \cK}{\partial z}(0,-\lambda),\\
    \lim_{d\to\infty}\bE{}{\cE_{\text{train}}(\hat{A}_t)} &= -h_t\mu_1^2\cK(0,-\lambda)-\lambda h_t\mu_1^2\frac{\partial \cK}{\partial z}(0,-\lambda)+1 \;,\nonumber
\end{align*}
where $\cK(q,z) =\frac{\psi_D\zeta_4}{\nu_1^2a_t^2\left(\frac{\mu_1^2}{\chi} + q\right)}$, and $(\zeta_1,\zeta_2,\zeta_3,\zeta_4)$ is a solution to the above self-consistent equations.
\end{theorem}
A proof of Theorem~\ref{thm:test_train} can be found in Appendix~\ref{appnxd_sec:proof_test_train}. Below, we give an short outline of the proof.\\

\begin{proof_outline}
The analysis relies on expressing the test and train errors as normalized traces of rational functions of random matrices. There are powerful random matrix techniques \cite{far_spectra_2006,bodin_random_2024} available to handle such situations, when the random matrices involved have Gaussian entries. However, in order to work with non-Gaussian feature vectors $\act{\rat{W_t}{d}\man{\rat{M}{D}\xi}}$, we rely on a \textit{deterministic equivalent} based approach \cite{couillet_random_2022}. A key aspect in our proof is that the effect of data distribution is captured by the first two joint moments of the vector $\rat{W}{d}\man{\rat{M}{D}\xi}$. To this end,  a CLT proved for such a scenario \citet{goldt_modeling_2020,hu_universality_2023}, gives us a control of its moments up to an $O(1/\sqrt{d})$ error. However, for the second moment our analysis requires a more precise $O(1/d)$ control of the error. We propose Conjecture~\ref{lemma:gep_sec_moment} to obtain the second moment at this level of accuracy, prove it for a weaker hypothesis.

Consequently, it turns out that all the quantities we need can be derived using the following function and its derivatives: 
$$K(q,z) = \frac{1}{d}\tr{\rat{W_t^T}{d}R(q,z)\rat{W_t}{d}}.$$ Now, it remains to obtain a set of self-consistent equations that gives asymptotic value of $K$. For this, we repeatedly use the asymptotic expression for the deterministic equivalents of sample covariance matrices (see Theorem~2.18 in \cite{couillet_random_2022}). This finally gives the set of equations in the Theorem, and the expressions for test and train errors. 
    
\end{proof_outline}



Theorem~\ref{thm:test_train} gives asymptotic expressions for test and train errors. However, a small test error alone does not guarantee an accurate approximation of the score function. To provide a meaningful reference, we introduce in the next section the test error associated with the exact score. This quantity serves as a baseline and will also play a important role in obtaining the score error, which directly measures the quality of score estimation.

\subsection{Test Error for Exact Score}
We characterize the test error associated with the exact score function $\nabla\log P_t$, defined as
\begin{equation}
        \cE_{\text{test}}^\ast = \frac{1}{d}\shortexpect_{x,z}{\norm{\sqrt{h_t}\nabla\log P_t(a_tx+\sqrt{h_t}z)+z}^2}.
    \end{equation}

To this end, we establish lemmas that connects $\cE_{\text{test}}^\ast$ to the free energy of a Generalized Linear Model. Lemma~\ref{lemma:norm_score} relates the squared norm of the score function to the minimum mean squared error (MMSE) of an associated Gaussian channel. Subsequently, Lemma~\ref{lemma:mmse_free_energy} expresses this MMSE in terms of the free energy of a Generalized linear model. 
\begin{lemma}\label{lemma:norm_score}
The test error for the exact score function can be expressed as
\begin{equation}\label{eqn:exact_score_formula}
        \cE_{\text{test}}^\ast = \frac{1}{d}\shortexpect_{x,z}{\norm{\sqrt{h_t}\nabla\log P_t(a_tx+\sqrt{h_t}z)+z}^2} = 1-\frac{h_t}{d}\shortexpect_{y\sim P_t}{\norm{\nabla\log P_t(y)}^2}.
    \end{equation}
    Moreover,
    \begin{equation}
        \frac{1}{d}\shortexpect_{y\sim P_t}{\norm{\nabla\log P_t(y)}^2} = \frac{1}{h_t}\left(1-\frac{a_t^2}{h_t}\frac{1}{d}\text{MMSE}(x|y)\right),
    \end{equation}
    where $\text{MMSE}(x|y) = \bE{X\sim P_0,\; Y = a_tX+\sqrt{h_t}Z}{\norm{X-\bE{}{X|Y}}^2}$.
\end{lemma}
\begin{lemma}\label{lemma:mmse_free_energy}
    Let $\eta = \frac{a_t}{\sqrt{h_t}}$. Then,
        $\frac{d}{d\eta}\bE{y\sim P_t}{\frac{1}{d}\log P_t(y)} = a_t\sqrt{h_t}-\frac{a_t}{\sqrt{h_t}}\frac{1}{d}\text{MMSE}(x|y)$. 
\end{lemma}
Proof of Lemma~\ref{lemma:norm_score} can be found in Appendix~\ref{apndx_sec:proof_lemma_norm_score}. Proof of Lemma~\ref{lemma:mmse_free_energy} follows directly by taking the derivative of l.h.s. w.r.t. $\eta$.

Lemmas~\ref{lemma:norm_score} and \ref{lemma:mmse_free_energy} suggests that given access to $\bE{y\sim P_t}{\frac{1}{d}\log P_t(y)}$, we can recover $\cE_{test}^\ast$. To compute $\bE{y\sim P_t}{\frac{1}{d}\log P_t(y)}$, we invoke the replica-symmetric formula for the partition function of Generalized Linear Models (GLMs). Specifically, consider an inference model with observations
$x = a_t\man{M\xi/\sqrt{D}}+\sqrt{h_t} z$, where $\xi$ is a signal to be estimated, $z$ is Gaussian additive noise, and $a_t^2/h_t= e^{-2t}/(1-e^{-2t})$ is the signal-to-noise ratio. This is a statistical mechanics spin-glass problem with Nishimori symmetry, whose 
rigorous theory was developed in~\cite{barbier_optimal_2019}. 
In this reference it is proved that:
\begin{align*}
    \lim_{D\to \infty} \frac{1}{D} \shortexpect_{y\sim P_t}\log P_t(y) &= \sup_{q\in[0,1]} \inf_{r\ge 0} f_{\text{RS}}(q,r)\eqdef f^\star(t),
    \label{eq:fstar}
\end{align*}
where
    $f_{\text{RS}}(q,r) = \phi(r) + \psi_D^{-1}\Psi(q)-rq/2 \;$,
    $\phi(r) = \shortexpect_{X_0,Z_0}{\log \int \dd w\frac{e^{-\frac{w^2}{2}}}{\sqrt{2\pi}}e^{rwX_0+\sqrt{r}xZ_0-rx^2/2}} \;$, and 
    $\Psi(q) = \shortexpect_{Y_0,V}{\log \int \dd w \frac{e^{-w^2/2}}{\sqrt{2\pi}}\frac{e^{-\frac{(Y_0-a_t\phi(\sqrt{q}V+\sqrt{1-q}w))^2}{2h_t}}}{\sqrt{2\pi h_t}}} \;$,
with $X_0\sim\cN{0,1}$, $Z_0,V,W,Z\sim\cN{0,1}$ and $Y_0 = a_t\man{\sqrt{q}V+\sqrt{1-q}W}+\sqrt{h_t}Z$. A straight-forward computation shows that $\phi(r) = \frac{r}{2}-\frac{1}{2}\log(1+r)$. For linear manifolds, the test error for the exact score can be computed in closed form and is given in Appendix~\ref{apndx_sec:exact_score_linear}.

Fig.~\ref{fig:test_train_pDt} displays the test and training errors predicted by Theorem~\ref{thm:test_train} for the case of ReLU activation and a linear manifold folding map. It also shows the test error for the exact score. We discuss the Figure in detail in Sec.~\ref{sec:discussion}.

\begin{figure}[t]
    \centering
    \includegraphics[width=0.95\linewidth]{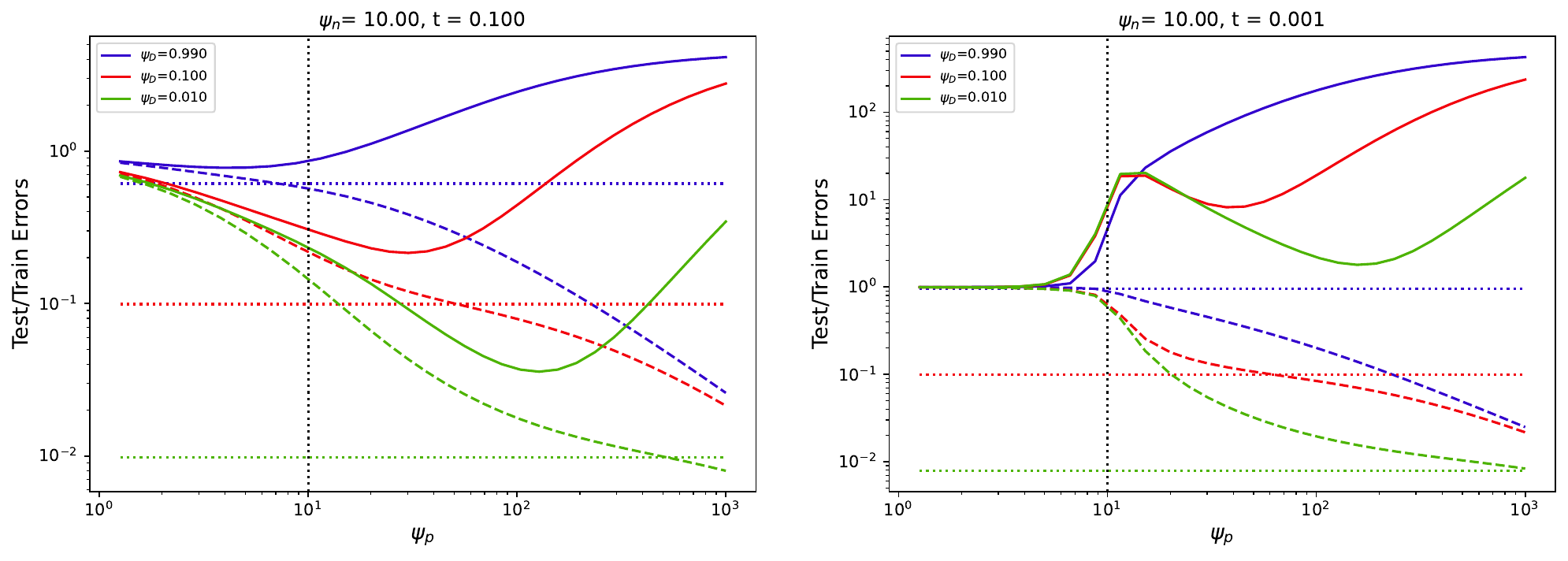}
    \caption{Test (solid lines) and train (dashed lines) errors for $\varrho =$ ReLU, $\sigma(x)=x$. Dotted horizontal lines denote the test error for exact score function obtained using \eqref{eqn:exact_score_formula}.}
    \label{fig:test_train_pDt}
\end{figure}

\subsection{Score error}
In Sec.~\ref{sec:test_train_RFNN}, we derived the test and train errors for the optimal RFNN score. However, these quantities alone do not fully characterize the performance of diffusion models. Under the idealized assumption that all other components of the diffusion pipeline are exact, the Kullback–Leibler divergence between the generated distribution and the target distribution can be expressed in terms of the score error \cite{song_maximum_2021}. More generally, for practical implementations of the reverse diffusion process, the score error provides upper bounds on various distances and divergences between the sampled and target distributions \citet{chen_sampling_2022,bortoli_convergence_2022}. Consequently, deriving the score error is essential for establishing theoretical guarantees on the generative performance of diffusion models.

Here we derive the score error from the test error and exact score error characterized in the previous section. Specifically, Lemma~\ref{lemma:test_err_decomp} establishes a bias–variance decomposition of the test error and expresses the variance term in terms of the expected squared norm of the exact score function. Its proof follows directly from orthogonality principle.

\begin{lemma}\label{lemma:test_err_decomp}
    For any learned score $\hat{s}$, the test error admits the following decomposition:
    \begin{equation*}
        \frac{1}{d}\shortexpect_{x,z}{\norm{\sqrt{h_t}\hat{s}(y_t)+z}^2} = \frac{h_t}{d}\shortexpect_{y_t\sim P_t}{\norm{\hat{s}(y_t)-\nabla\log P_t(y)}^2}+\frac{1}{d}\shortexpect_{x,z}{\norm{\sqrt{h_t}\nabla\log P_t(y_t)+z}^2},
    \end{equation*}
    where $y_t = a_tx+\sqrt{h_t}z$. That is, $\cE_{test} = h_t\cE_{score}+\cE_{test}^\ast$.
\end{lemma}  
Fig.~\ref{fig:score_pDt} displays the resulting score error for ReLU activation, and (non)-linear manifold models. 

\subsection{Discussion}\label{sec:discussion}
We discuss in detail the learning curves derived from the results of the previous sections. We start by reviewing relevant aspects of denoising score matching.

 Consider the loss function given in (\ref{eqn:dsm_minf}). It has an unique minimizer given by:
\begin{equation}\label{eqn:emp_opt_score}
    s^e(t,x) = \left(\sum_{i=1}^n-\left(\frac{x-a_tx_i}{h_t}\right)e^{-\frac{\norm{x-a_tx_i}^2}{2h_t}}\right)\left(\sum_{i=1}^ne^{-\frac{\norm{x-a_tx_i}^2}{2h_t}}\right)^{-1} \;.
\end{equation}
The score $s^e$ is often referred to as the \textit{empirical optimal score}. A backward process using $s^e$ converges in distribution to the empirical distribution of the dataset $\cS$ as $t\rightarrow 0$. That is, the backward process collapses to one of the data samples as $t\rightarrow 0$. Note that $s^e$ is a softmax function, and thus, if the learned score is close to the empirical optimal score, the learned score evaluated at $x_t$ would be approximately $-\frac{x_t-a_t\bar{x}}{h_t}$ where $\bar{x}$ is the closest data sample to $x_t$. 

Next, for any $\hat{s}$, we leverage the optimality of $s^e$ to decompose the training error as follows:
\begin{equation*}
    \frac{1}{dn}\sum_{i=1}^{n}\shortexpect_z{\norm{\sqrt{h_t}\hat{s}(y_i)+z}^2} = h_t\frac{1}{dn}\sum_{i=1}^{n}{\shortexpect_{z}{\norm{\hat{s}(y_i)-s^e(t,y_i}^2}} + \frac{1}{dn}\sum_{i=1}^{n}{\shortexpect_{z}{\norm{\sqrt{h_t}s^e(t, y_i)+z}^2}}\;,
\end{equation*}
where $y_i = a_tx_i+\sqrt{h_t}z$.
Note that the second term on the right-hand side does not depend on $\hat{s}$. Therefore, minimizing the training loss over $\hat{s}$ is equivalent to minimizing only the first term. A central design question in diffusion models is the choice of an appropriate function class such that the minimizer of this first term {\it within the class} closely approximates the exact score. This is a delicate issue: a poor choice of function class can result either in memorization of the training data or in low-quality generated samples.

The score error, on the other hand, directly quantifies the quality of approximation with respect to the exact score. Ideally, we seek a regime in which the score error remains small while the training error is minimized. 


Next, we discuss some aspects regarding the exact score and empirical optimal score when data lies on a linear manifold. Suppose $\Pi\in\R^{d\times D}$ is a matrix with orthonormal columns and let $\Pi^\perp$ its orthogonal complement. Let the data lie on linear manifold given by $x=\Pi \xi$. Then, the exact score function evaluated at $x$ is given by
\begin{align*}
s^\ast(x) &= -\left(a_t^2\frac{\Pi\Pi^T}{d}+h_tI\right)^{-1}x =  -\left(\Pi\Pi^T+h_t\Pi^\perp{\Pi^\perp}^T\right)^{-1}x =  -\left(\Pi\Pi^T+\frac{1}{h_t}\Pi^\perp{\Pi^\perp}^T\right)x.
\end{align*}
Thus, the magnitude of the exact score is $O(1)$ w.r.t. $t$ in the directions parallel to $\Pi$, while it's $O(\frac{1}{h_t})$ in the orthogonal directions. 
Now, suppose the learned score function approximates the empirical optimal score \eqref{eqn:emp_opt_score}. Then, at $x_t = a_t\Pi\xi+\sqrt{h_t}z$, let's compute its component along $\Pi^\perp$. We have (using the remarks after \eqref{eqn:emp_opt_score}) $\Pi^\perp \hat{s}(a_t\Pi\xi+\sqrt{h_t}z)\approx -\Pi^\perp \frac{a_t\Pi\xi+\sqrt{h_t}z-a_t\Pi\bar{\xi}}{h_t} = -\frac{\Pi^\perp z}{\sqrt{h_t}}$, where $\Pi\bar{\xi}$ is the nearest data sample to $x_t$. This is precisely the component of the exact score along $\Pi^\perp$ at $x_t$. {\it This means that, even when the learned score approximates the empirical optimal score, the orthogonal component can generalize}. 

Under this premise, we analyze the learning curves for linear manifold data displayed in Figs.~\ref{fig:test_train_pDt} and \ref{fig:score_lin_pDt}. 
\begin{figure}
\centering
\subfigure[$\varrho =$ ReLU, $\sigma(x)=x$.]{\includegraphics[width=0.48\textwidth]{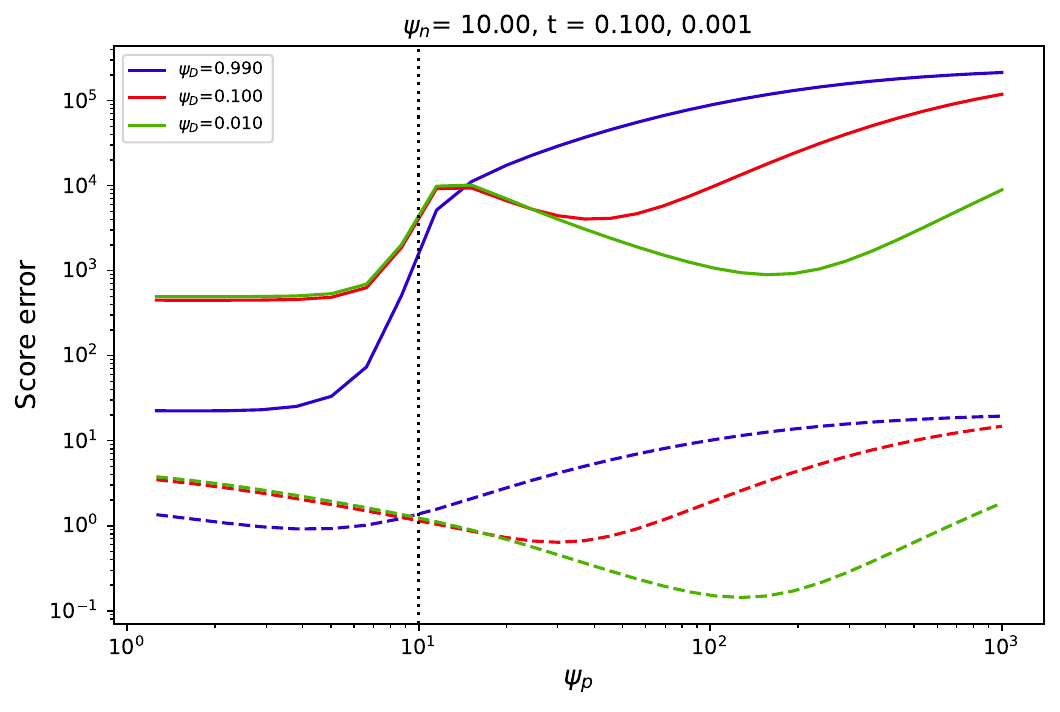}
\label{fig:score_lin_pDt}}
\subfigure[$\varrho =$ ReLU, $\sigma(x)=0.975*x+0.223*\frac{x^2-1}{\sqrt{2}}$]{\includegraphics[width=0.48\textwidth]{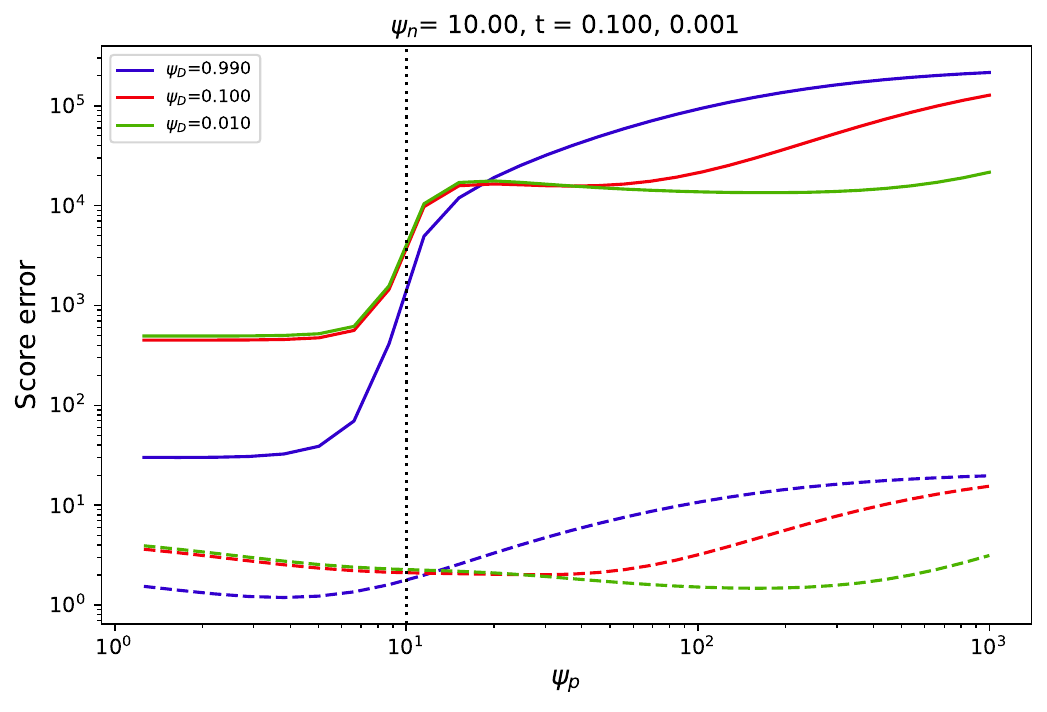}
\label{fig:score_nonlin_pDt}}
\caption{Score errors for linear and non-linear manifolds. Solid lines are for $t=0.001$, and dashed lines are for $t=0.1$.}
\label{fig:score_pDt}
\end{figure}
Fig~\ref{fig:test_train_pDt} displays the test and train errors for $\varrho =$ RELU, $\sigma(x)=x$, as a function of $\psi_p$, for different $\psi_D$ and $t$ and a fixed $\psi_n = 10.0$. Corresponding score errors are shown in Fig.~\ref{fig:score_lin_pDt}. The error curves exhibit several interesting behaviors which we elucidate here. First, we consider the {\it small} $t$ regime ($t=0.001$, Figs.~\ref{fig:test_train_pDt},~\ref{fig:score_lin_pDt}) since the characteristics are sharper here. Below we give explanations for the behaviors observed for different regimes of $\psi_p$: 
\begin{enumerate}
    \item $\psi_p<\psi_n$: In this regime, the learned score generalizes well, meaning that it closely approximates the exact score. This is evidenced by the agreement between the test and training errors. Because the magnitude of the exact score $s^*(x)$ is large along the orthogonal directions, the corresponding score error is also large in those directions. However, for large $\psi_D$, the score error decreases due to the reduced number of orthogonal directions.
    \item $\psi_p\approx \psi_n$: This marks the threshold at which the error curves transition from generalization to memorization. The onset of memorization is evidenced by the decreasing training error, accompanied by a rapid increase in the score error.
    \item $\psi_p>\psi_n$: This is an interesting regime in which the behavior depends on the dimensionality of the manifold. Two competing sources contribute to the score error. On one hand, as the learned score better approximates the empirical optimal score, its orthogonal component becomes closer to that of the exact score, which tends to reduce the overall error. On the other hand, the parallel component departs from the exact score, thereby increasing the error. The observed score error reflects the balance between these opposing effects. The ultimate trend depends on the manifold dimension: lower-dimensional manifolds have more orthogonal directions, which favors a smaller score error.
    \item $\psi_p\gg \psi_n$: In this regime, the error arising from the parallel components dominates that from the orthogonal components, leading to an overall increase in the score error.
\end{enumerate}
For {\it larger} $t$ ($t=0.1$, Figs.~\ref{fig:test_train_pDt},~\ref{fig:score_lin_pDt}), the qualitative behavior of the learning curves remains similar to that observed for small $t$ when $\psi_p>\psi_n$. However, the generalization to memorization phase transition at $\psi_p=\psi_n$ becomes smoother. This effect is due to the fact that larger $t$ introduces an \emph{implicit regularization}. This is evident from the proof of Theorem~\ref{thm:test_train}, where the spectral parameter of the resolvent $R$ \eqref{apndx_eqnl:resolvent} takes the form $s^2+\lambda$. Since the value of $s^2= 1-c(a_t^2) - h_t\mu_1^2$ increases with $t$, the effective regularization strength also increases.
\begin{figure}[t]
\centering
\subfigure[Score error vs. $\psi_n$: Solid lines are for $t=0.001$, and dashed lines are for $t=0.1$.]{\includegraphics[width=0.48\textwidth]{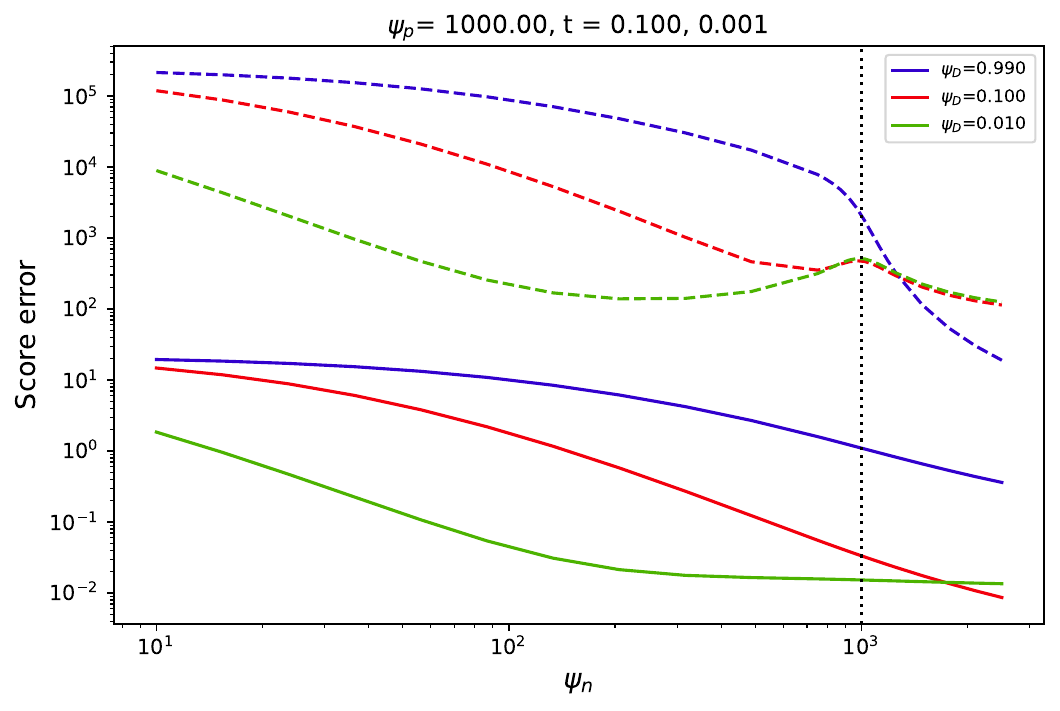}
\label{fig:score_nDt}}
\subfigure[Sample complexity for $\epsilon=0.2$ (solid lines), and $\epsilon=0.05$ (dashed lines).]{\includegraphics[width=0.48\textwidth]{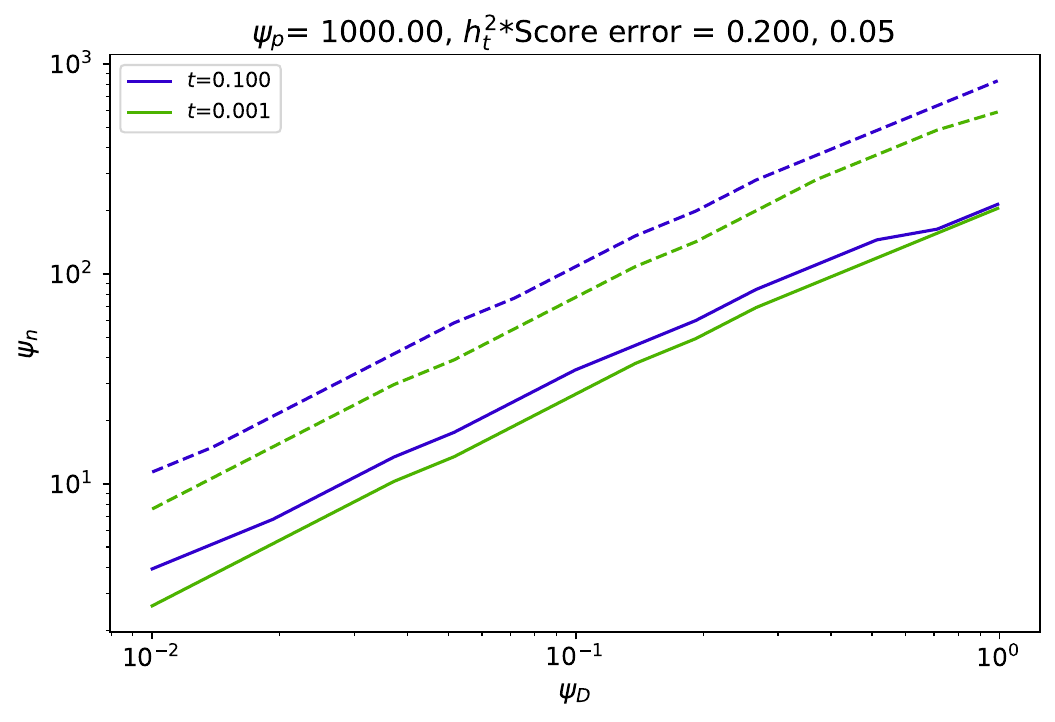}
\label{fig:sample_cmplxty_Dte}}
\caption{Score error and sample complexity for $\varrho =$ ReLU, $\sigma(x)=x$.}
\end{figure}

\paragraph{Sample complexity} Fig~\ref{fig:score_nDt} shows the score errors for $\varrho =$ RELU, $\sigma(x)=x$, as a function of $\psi_n$ for different $\psi_D$ and $t$ and a fixed $\psi_p=1000.0$. For linear manifolds, we observe that the error due to memorization is smaller as $\psi_D$ decreases for a fixed $\psi_n$. This can be made more quantitative by defining a notion of sample complexity. As shown in \cite{bortoli_convergence_2022}, a uniform bound on $h_t^2$ times the score error is sufficient to obtain overall theoretical guarantees for diffusion models. Accordingly, Fig.~\ref{fig:sample_cmplxty_Dte} reports the minimum value of $\psi_n$ required for $h_t^2$ times the score error to fall below a prescribed threshold $\epsilon$. We refer to this quantity as the \emph{sample complexity}. We observe that the sample complexity grows approximately linearly with $\psi_D$, indicating that it is governed by the intrinsic dimension of the data manifold rather than the ambient dimension.

\paragraph{Non-linear manifolds} We now briefly discuss the score curves for a nonlinear manifold setting. Specifically, we consider a manifold folding function given by $\sigma(x)=0.975*x+0.223*\frac{x^2-1}{\sqrt{2}}$, which translates to a $5\%$ non-linear component in terms of $L_2$ power w.r.t. Gaussian measure. Figure~\ref{fig:score_nonlin_pDt} shows the resulting score error in this setting. We observe that the phase in which the score error decreases as $\psi_p$ increases beyond $\psi_n$ begins to disappear, indicating that the low-dimensional structure becomes less effective at reducing the score error.

This behavior can be explained as follows. A careful examination of the proof of Theorem~\ref{thm:test_train} reveals that the nonlinearity interpolates between two limiting cases: data supported on a low-dimensional linear manifold and isotropic data in the ambient space. When the nonlinearity power ($1-\nu_1^2$ in Theorem~\ref{thm:test_train}) is zero, the data lie on a linear manifold; when it equals one, the data follow an isotropic Gaussian distribution in the ambient space. Consequently, as shown in Fig.~\ref{fig:score_nonlin_pDt}, introducing nonlinearity causes the curves corresponding to different $\psi_D$ to move closer together.
\begin{wrapfigure}{r}{0.5\textwidth}
    \centering
\includegraphics[width=0.47\textwidth]{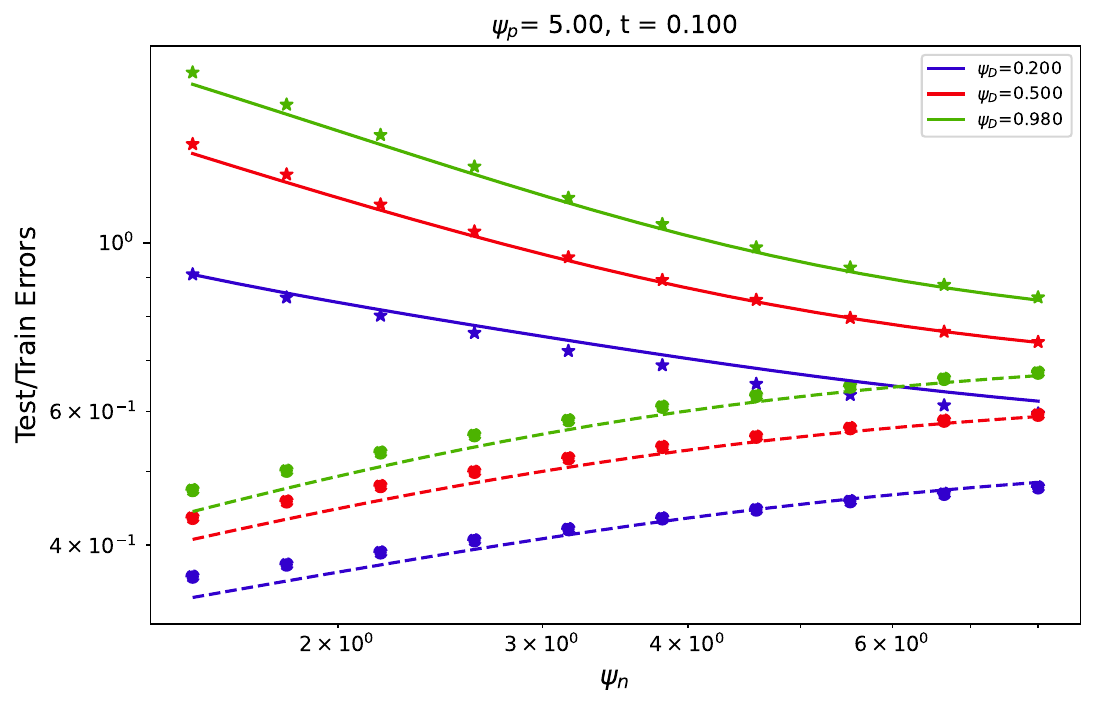}
    \caption{Comparison of test (solid lines) and train (dashed lines) errors with numerical simulations (points) for $\varrho =$~ReLU, $\sigma=$~tanh.}
    \label{fig:test_train_sim}
    \vspace{-70pt}
\end{wrapfigure}

\paragraph{Numerical Verification}
Finally Fig.~\ref{fig:test_train_sim} shows the comparison of numerically obtained  test and train errors against the theoretical curves obtained using Theorem~\ref{thm:test_train}.


\section{Conclusion and Future Work}
We studied the problem of learning the score function in diffusion models when the data distribution is supported on a low-dimensional manifold. Focusing on denoising score matching with a random feature neural network parameterization, we derived asymptotically exact expressions for the test, train, and score errors in the high-dimensional limit.

A central conclusion of our study is that for linear manifolds, the sample complexity required to learn the score function scales linearly with the intrinsic dimension of the underlying manifold, rather than with the ambient dimension. To some extent, this result offers a theoretical insights into the empirical effectiveness of diffusion models in high-dimensional settings. From a theoretical perspective, it highlights denoising score matching as a mechanism that exploits geometry of data, even when learning is performed in the full ambient space. However, we also discovered that as the non-linearity of the manifold increases, the situation becomes more similar to having a distribution supported in the ambient space, and the benefits of the low-dimensional structure diminishes.

There are several promising directions for future work. One natural extension is to move beyond random feature models and analyze fully trained neural networks, where feature learning may further adapt to manifold structure. Another direction is to study different noise schedules, including non-isotropic or data-dependent perturbations, and to understand how these choices affect sample complexity and generalization. 

\subsubsection*{Acknowledgements}
The work of A. J. G. has been supported by Swiss National Science Foundation grant number 200021-204119. 
\bibliography{references.bib} 

@misc{azangulov_convergence_2025,
    title = {Convergence of {Diffusion} {Models} {Under} the {Manifold} {Hypothesis} in {High}-{Dimensions}},
    url = {http://arxiv.org/abs/2409.18804},
    doi = {10.48550/arXiv.2409.18804},
    urldate = {2026-01-05},
    publisher = {arXiv},
    author = {Azangulov, Iskander and Deligiannidis, George and Rousseau, Judith},
    month = apr,
    year = {2025},
    note = {arXiv:2409.18804 [stat]},
}

@book{janson_gaussian_1997,
	address = {Cambridge},
	series = {Cambridge {Tracts} in {Mathematics}},
	title = {Gaussian {Hilbert} {Spaces}},
	isbn = {978-0-521-56128-0},
	url = {https://www.cambridge.org/core/books/gaussian-hilbert-spaces/658C87D5A0E7FB440FC34D82B08167FC},
	doi = {10.1017/CBO9780511526169},
	urldate = {2026-06-12},
	publisher = {Cambridge University Press},
	author = {Janson, Svante},
	year = {1997},
}

@article{chou_equilibrium_1985,
	title = {Equilibrium and nonequilibrium formalisms made unified},
	volume = {118},
	issn = {0370-1573},
	url = {https://www.sciencedirect.com/science/article/pii/037015738590136X},
	doi = {10.1016/0370-1573(85)90136-X},
	number = {1},
	urldate = {2026-05-18},
	journal = {Physics Reports},
	author = {Chou, Kuang-chao and Su, Zhao-bin and Hao, Bai-lin and Yu, Lu},
	month = feb,
	year = {1985},
	pages = {1--131},
}

@article{slepian_symmetrized_1972,
	title = {On the {Symmetrized} {Kronecker} {Power} of a {Matrix} and {Extensions} of {Mehler}’s {Formula} for {Hermite} {Polynomials}},
	volume = {3},
	issn = {0036-1410},
	url = {https://epubs.siam.org/doi/abs/10.1137/0503060},
	doi = {10.1137/0503060},
	number = {4},
	urldate = {2026-05-18},
	journal = {SIAM Journal on Mathematical Analysis},
	publisher = {Society for Industrial and Applied Mathematics},
	author = {Slepian, David},
	month = nov,
	year = {1972},
	pages = {606--616},
}

@misc{bodin_gradient_2022,
	title = {Gradient flow in the gaussian covariate model: exact solution of learning curves and multiple descent structures},
	shorttitle = {Gradient flow in the gaussian covariate model},
	url = {http://arxiv.org/abs/2212.06757},
	doi = {10.48550/arXiv.2212.06757},
	urldate = {2026-02-04},
	publisher = {arXiv},
	author = {Bodin, Antoine and Macris, Nicolas},
	month = dec,
	year = {2022},
	note = {arXiv:2212.06757 [stat]},
}

@book{couillet_random_2022,
	edition = {1},
	title = {Random {Matrix} {Methods} for {Machine} {Learning}},
	copyright = {https://www.cambridge.org/core/terms},
	isbn = {978-1-009-12849-0 978-1-009-12323-5},
	doi = {10.1017/9781009128490},
	language = {en},
	urldate = {2024-10-07},
	publisher = {Cambridge University Press},
	author = {Couillet, Romain and Liao, Zhenyu},
	month = jul,
	year = {2022},
}

@inproceedings{george_analysis_2025,
	title = {Analysis of {Diffusion} {Models} for {Manifold} {Data}},
	issn = {2157-8117},
	doi = {10.1109/ISIT63088.2025.11195641},
	urldate = {2026-02-04},
	booktitle = {2025 {IEEE} {International} {Symposium} on {Information} {Theory} ({ISIT})},
	author = {George, Anand Jerry and Veiga, Rodrigo and Macris, Nicolas},
	month = jun,
	year = {2025},
	note = {ISSN: 2157-8117},
	pages = {1--6},
}

@misc{erdos_matrix_2019,
	title = {The matrix {Dyson} equation and its applications for random matrices},
	url = {http://arxiv.org/abs/1903.10060},
	doi = {10.48550/arXiv.1903.10060},
	urldate = {2025-10-22},
	publisher = {arXiv},
	author = {Erdos, Laszlo},
	month = mar,
	year = {2019},
	note = {arXiv:1903.10060 [math]},
}

@misc{george_denoising_2025,
	title = {Denoising {Score} {Matching} with {Random} {Features}: {Insights} on {Diffusion} {Models} from {Precise} {Learning} {Curves}},
	shorttitle = {Denoising {Score} {Matching} with {Random} {Features}},
	url = {http://arxiv.org/abs/2502.00336},
	doi = {10.48550/arXiv.2502.00336},
	urldate = {2026-01-23},
	publisher = {arXiv},
	author = {George, Anand Jerry and Veiga, Rodrigo and Macris, Nicolas},
	month = oct,
	year = {2025},
	note = {arXiv:2502.00336 [cs]},
}

@article{mei_generalization_2022,
	title = {The {Generalization} {Error} of {Random} {Features} {Regression}: {Precise} {Asymptotics} and the {Double} {Descent} {Curve}},
	volume = {75},
	shorttitle = {The {Generalization} {Error} of {Random} {Features} {Regression}},
	language = {en},
	number = {4},
	urldate = {2025-01-13},
	journal = {Communications on Pure and Applied Mathematics},
	author = {Mei, Song and Montanari, Andrea},
	year = {2022},
	pages = {667--766},
}

@misc{zeno_when_2025,
	title = {When {Diffusion} {Models} {Memorize}: {Inductive} {Biases} in {Probability} {Flow} of {Minimum}-{Norm} {Shallow} {Neural} {Nets}},
	shorttitle = {When {Diffusion} {Models} {Memorize}},
	url = {http://arxiv.org/abs/2506.19031},
	doi = {10.48550/arXiv.2506.19031},
	urldate = {2025-06-28},
	publisher = {arXiv},
	author = {Zeno, Chen and Manor, Hila and Ongie, Greg and Weinberger, Nir and Michaeli, Tomer and Soudry, Daniel},
	month = jun,
	year = {2025},
	note = {arXiv:2506.19031 [stat]},
}

@misc{saha_generalization_2025,
	title = {Generalization {Bound} for {Diffusion} {Models} using {Random} {Features}},
	url = {http://arxiv.org/abs/2310.04417},
	doi = {10.48550/arXiv.2310.04417},
	urldate = {2025-09-28},
	publisher = {arXiv},
	author = {Saha, Esha and Tran, Giang},
	month = aug,
	year = {2025},
	note = {arXiv:2310.04417 [stat]},
}

@misc{bonnaire_why_2025,
	title = {Why {Diffusion} {Models} {Don}'t {Memorize}: {The} {Role} of {Implicit} {Dynamical} {Regularization} in {Training}},
	shorttitle = {Why {Diffusion} {Models} {Don}'t {Memorize}},
	url = {http://arxiv.org/abs/2505.17638},
	doi = {10.48550/arXiv.2505.17638},
	urldate = {2025-09-28},
	publisher = {arXiv},
	author = {Bonnaire, Tony and Urfin, Raphaël and Biroli, Giulio and Mézard, Marc},
	month = may,
	year = {2025},
	note = {arXiv:2505.17638 [cs]},
}

@inproceedings{ross_geometric_2024,
	title = {A {Geometric} {Framework} for {Understanding} {Memorization} in {Generative} {Models}},
	language = {en},
	urldate = {2025-05-15},
	booktitle = {The {Thirteenth} {International} {Conference} on {Learning} {Representations}},
	author = {Ross, Brendan Leigh and Kamkari, Hamidreza and Wu, Tongzi and Hosseinzadeh, Rasa and Liu, Zhaoyan and Stein, George and Cresswell, Jesse C. and Loaiza-Ganem, Gabriel},
	month = oct,
	year = {2024},
}

@article{li_generalization_2023,
	title = {On the {Generalization} {Properties} of {Diffusion} {Models}},
	volume = {36},
	language = {en},
	urldate = {2025-01-26},
	journal = {Advances in Neural Information Processing Systems},
	author = {Li, Puheng and Li, Zhong and Zhang, Huishuai and Bian, Jiang},
	month = dec,
	year = {2023},
	pages = {2097--2127},
}

@misc{achilli_memorization_2025,
	title = {Memorization and {Generalization} in {Generative} {Diffusion} under the {Manifold} {Hypothesis}},
	doi = {10.48550/arXiv.2502.09578},
	urldate = {2025-02-19},
	publisher = {arXiv},
	author = {Achilli, Beatrice and Ambrogioni, Luca and Lucibello, Carlo and Mézard, Marc and Ventura, Enrico},
	month = feb,
	year = {2025},
	note = {arXiv:2502.09578 [cond-mat]},
}

@inproceedings{pidstrigach_score-based_2022,
	title = {Score-{Based} {Generative} {Models} {Detect} {Manifolds}},
	volume = {35},
	booktitle = {Advances in {Neural} {Information} {Processing} {Systems}},
	publisher = {Curran Associates, Inc.},
	author = {Pidstrigach, Jakiw},
	year = {2022},
	pages = {35852--35865},
}

@article{goldt_modeling_2020,
	title = {Modeling the {Influence} of {Data} {Structure} on {Learning} in {Neural} {Networks}: {The} {Hidden} {Manifold} {Model}},
	volume = {10},
	issn = {2160-3308},
	shorttitle = {Modeling the {Influence} of {Data} {Structure} on {Learning} in {Neural} {Networks}},
	doi = {10.1103/PhysRevX.10.041044},
	language = {en},
	number = {4},
	urldate = {2025-01-22},
	journal = {Physical Review X},
	author = {Goldt, Sebastian and Mézard, Marc and Krzakala, Florent and Zdeborová, Lenka},
	month = dec,
	year = {2020},
	pages = {041044},
}

@article{barbier_optimal_2019,
	title = {Optimal errors and phase transitions in high-dimensional generalized linear models},
	volume = {116},
	issn = {0027-8424, 1091-6490},
	doi = {10.1073/pnas.1802705116},
	language = {en},
	number = {12},
	urldate = {2025-01-22},
	journal = {Proceedings of the National Academy of Sciences},
	author = {Barbier, Jean and Krzakala, Florent and Macris, Nicolas and Miolane, Léo and Zdeborová, Lenka},
	month = mar,
	year = {2019},
	pages = {5451--5460},
}

@misc{cui_precise_2025,
	title = {A precise asymptotic analysis of learning diffusion models: theory and insights},
	shorttitle = {A precise asymptotic analysis of learning diffusion models},
	doi = {10.48550/arXiv.2501.03937},
	urldate = {2025-01-13},
	publisher = {arXiv},
	author = {Cui, Hugo and Pehlevan, Cengiz and Lu, Yue M.},
	month = jan,
	year = {2025},
	note = {arXiv:2501.03937 [cs]},
}

@article{ambrogioni_search_2024,
	title = {In {Search} of {Dispersed} {Memories}: {Generative} {Diffusion} {Models} {Are} {Associative} {Memory} {Networks}},
	volume = {26},
	copyright = {https://creativecommons.org/licenses/by/4.0/},
	issn = {1099-4300},
	shorttitle = {In {Search} of {Dispersed} {Memories}},
	doi = {10.3390/e26050381},
	language = {en},
	number = {5},
	urldate = {2025-01-28},
	journal = {Entropy},
	author = {Ambrogioni, Luca},
	month = apr,
	year = {2024},
	pages = {381},
}

@article{vincent_connection_2011,
	title = {A {Connection} {Between} {Score} {Matching} and {Denoising} {Autoencoders}},
	volume = {23},
	issn = {0899-7667, 1530-888X},
	doi = {10.1162/NECO_a_00142},
	language = {en},
	number = {7},
	urldate = {2025-01-18},
	journal = {Neural Computation},
	author = {Vincent, Pascal},
	month = jul,
	year = {2011},
	pages = {1661--1674},
}

@inproceedings{song_score-based_2020,
	title = {Score-{Based} {Generative} {Modeling} through {Stochastic} {Differential} {Equations}},
	language = {en},
	urldate = {2025-01-28},
	booktitle = {International {Conference} on {Learning} {Representations}},
	author = {Song, Yang and Sohl-Dickstein, Jascha and Kingma, Diederik P. and Kumar, Abhishek and Ermon, Stefano and Poole, Ben},
	month = oct,
	year = {2020},
}

@inproceedings{song_maximum_2021,
	title = {Maximum {Likelihood} {Training} of {Score}-{Based} {Diffusion} {Models}},
	volume = {34},
	booktitle = {Advances in {Neural} {Information} {Processing} {Systems}},
	publisher = {Curran Associates, Inc.},
	author = {Song, Yang and Durkan, Conor and Murray, Iain and Ermon, Stefano},
	year = {2021},
	pages = {1415--1428},
}

@inproceedings{raya_spontaneous_2023,
	title = {Spontaneous symmetry breaking in generative diffusion models},
	volume = {36},
	booktitle = {Advances in {Neural} {Information} {Processing} {Systems}},
	publisher = {Curran Associates, Inc.},
	author = {Raya, Gabriel and Ambrogioni, Luca},
	year = {2023},
	pages = {66377--66389},
}

@inproceedings{nichol_glide_2022,
	title = {{GLIDE}: {Towards} {Photorealistic} {Image} {Generation} and {Editing} with {Text}-{Guided} {Diffusion} {Models}},
	issn = {2640-3498},
	shorttitle = {{GLIDE}},
	language = {en},
	urldate = {2025-01-26},
	booktitle = {Proceedings of the 39th {International} {Conference} on {Machine} {Learning}},
	publisher = {PMLR},
	author = {Nichol, Alexander Quinn and Dhariwal, Prafulla and Ramesh, Aditya and Shyam, Pranav and Mishkin, Pamela and Mcgrew, Bob and Sutskever, Ilya and Chen, Mark},
	month = jun,
	year = {2022},
	pages = {16784--16804},
}

@article{kibble_extension_1945,
	title = {An extension of a theorem of {Mehler}'s on {Hermite} polynomials},
	volume = {41},
	copyright = {https://www.cambridge.org/core/terms},
	issn = {0305-0041, 1469-8064},
	doi = {10.1017/S0305004100022313},
	language = {en},
	number = {1},
	urldate = {2025-01-21},
	journal = {Mathematical Proceedings of the Cambridge Philosophical Society},
	author = {Kibble, W. F.},
	month = jun,
	year = {1945},
	pages = {12--15},
}

@inproceedings{kadkhodaie_generalization_2023,
	title = {Generalization in diffusion models arises from geometry-adaptive harmonic representations},
	language = {en},
	urldate = {2025-01-28},
	booktitle = {The {Twelfth} {International} {Conference} on {Learning} {Representations}},
	author = {Kadkhodaie, Zahra and Guth, Florentin and Simoncelli, Eero P. and Mallat, Stéphane},
	month = oct,
	year = {2023},
}

@inproceedings{ho_classifier-free_2021,
	title = {Classifier-{Free} {Diffusion} {Guidance}},
	language = {en},
	urldate = {2025-01-26},
	booktitle = {{NeurIPS} 2021 {Workshop} on {Deep} {Generative} {Models} and {Downstream} {Applications}},
	author = {Ho, Jonathan and Salimans, Tim},
	month = dec,
	year = {2021},
}

@inproceedings{goldt_gaussian_2022,
	title = {The {Gaussian} equivalence of generative models for learning with shallow neural networks},
	issn = {2640-3498},
	language = {en},
	urldate = {2025-01-22},
	booktitle = {Proceedings of the 2nd {Mathematical} and {Scientific} {Machine} {Learning} {Conference}},
	publisher = {PMLR},
	author = {Goldt, Sebastian and Loureiro, Bruno and Reeves, Galen and Krzakala, Florent and Mezard, Marc and Zdeborova, Lenka},
	month = apr,
	year = {2022},
	pages = {426--471},
}

@misc{far_spectra_2006,
	title = {Spectra of large block matrices},
	urldate = {2024-10-23},
	publisher = {arXiv},
	author = {Far, Reza Rashidi and Oraby, Tamer and Bryc, Wlodzimierz and Speicher, Roland},
	month = oct,
	year = {2006},
	note = {arXiv:cs/0610045},
}

@inproceedings{cui_analysis_2023,
	title = {Analysis of {Learning} a {Flow}-based {Generative} {Model} from {Limited} {Sample} {Complexity}},
	language = {en},
	urldate = {2025-01-28},
	booktitle = {The {Twelfth} {International} {Conference} on {Learning} {Representations}},
	author = {Cui, Hugo and Krzakala, Florent and Vanden-Eijnden, Eric and Zdeborova, Lenka},
	month = oct,
	year = {2023},
}

@inproceedings{chen_sampling_2022,
	title = {Sampling is as easy as learning the score: theory for diffusion models with minimal data assumptions},
	shorttitle = {Sampling is as easy as learning the score},
	language = {en},
	urldate = {2025-01-28},
	booktitle = {The {Eleventh} {International} {Conference} on {Learning} {Representations}},
	author = {Chen, Sitan and Chewi, Sinho and Li, Jerry and Li, Yuanzhi and Salim, Adil and Zhang, Anru},
	month = sep,
	year = {2022},
}

@inproceedings{chen_score_2023,
	title = {Score {Approximation}, {Estimation} and {Distribution} {Recovery} of {Diffusion} {Models} on {Low}-{Dimensional} {Data}},
	issn = {2640-3498},
	language = {en},
	urldate = {2025-01-28},
	booktitle = {Proceedings of the 40th {International} {Conference} on {Machine} {Learning}},
	publisher = {PMLR},
	author = {Chen, Minshuo and Huang, Kaixuan and Zhao, Tuo and Wang, Mengdi},
	month = jul,
	year = {2023},
	pages = {4672--4712},
}

@phdthesis{bodin_random_2024,
	title = {Random matrix methods for high-dimensional machine learning models},
	doi = {10.5075/epfl-thesis-10524},
	language = {en},
	urldate = {2025-01-24},
	school = {EPFL},
	author = {Bodin, Antoine Philippe Michel},
	year = {2024},
}

@inproceedings{bodin_model_2021,
	title = {Model, sample, and epoch-wise descents: exact solution of gradient flow in the random feature model},
	volume = {34},
	shorttitle = {Model, sample, and epoch-wise descents},
	urldate = {2025-01-24},
	booktitle = {Advances in {Neural} {Information} {Processing} {Systems}},
	publisher = {Curran Associates, Inc.},
	author = {Bodin, Antoine and Macris, Nicolas},
	year = {2021},
	pages = {21605--21617},
}

@inproceedings{benton_nearly_2023,
	title = {Nearly d-{Linear} {Convergence} {Bounds} for {Diffusion} {Models} via {Stochastic} {Localization}},
	language = {en},
	urldate = {2025-01-28},
	booktitle = {The {Twelfth} {International} {Conference} on {Learning} {Representations}},
	author = {Benton, Joe and Bortoli, Valentin De and Doucet, Arnaud and Deligiannidis, George},
	month = oct,
	year = {2023},
}

@misc{achilli_losing_2024,
	title = {Losing dimensions: {Geometric} memorization in generative diffusion},
	shorttitle = {Losing dimensions},
	urldate = {2024-11-17},
	publisher = {arXiv},
	author = {Achilli, Beatrice and Ventura, Enrico and Silvestri, Gianluigi and Pham, Bao and Raya, Gabriel and Krotov, Dmitry and Lucibello, Carlo and Ambrogioni, Luca},
	month = oct,
	year = {2024},
	note = {arXiv:2410.08727},
}

@inproceedings{wang_evaluating_2024,
	title = {Evaluating the design space of diffusion-based generative models},
	language = {en},
	urldate = {2025-01-15},
	booktitle = {The {Thirty}-eighth {Annual} {Conference} on {Neural} {Information} {Processing} {Systems}},
	author = {Wang, Yuqing and He, Ye and Tao, Molei},
	month = nov,
	year = {2024},
}

@inproceedings{han_neural_2023,
	title = {Neural {Network}-{Based} {Score} {Estimation} in {Diffusion} {Models}: {Optimization} and {Generalization}},
	shorttitle = {Neural {Network}-{Based} {Score} {Estimation} in {Diffusion} {Models}},
	language = {en},
	urldate = {2025-01-17},
	booktitle = {The {Twelfth} {International} {Conference} on {Learning} {Representations}},
	author = {Han, Yinbin and Razaviyayn, Meisam and Xu, Renyuan},
	month = oct,
	year = {2023},
}

@inproceedings{song_generative_2019,
	title = {Generative {Modeling} by {Estimating} {Gradients} of the {Data} {Distribution}},
	volume = {32},
	urldate = {2024-05-22},
	booktitle = {Advances in {Neural} {Information} {Processing} {Systems}},
	publisher = {Curran Associates, Inc.},
	author = {Song, Yang and Ermon, Stefano},
	year = {2019},
}

@inproceedings{sohl-dickstein_deep_2015,
	title = {Deep {Unsupervised} {Learning} using {Nonequilibrium} {Thermodynamics}},
	issn = {1938-7228},
	language = {en},
	urldate = {2025-01-10},
	booktitle = {Proceedings of the 32nd {International} {Conference} on {Machine} {Learning}},
	publisher = {PMLR},
	author = {Sohl-Dickstein, Jascha and Weiss, Eric and Maheswaranathan, Niru and Ganguli, Surya},
	month = jun,
	year = {2015},
	pages = {2256--2265},
}

@article{shah_learning_2023,
	title = {Learning {Mixtures} of {Gaussians} {Using} the {DDPM} {Objective}},
	volume = {36},
	language = {en},
	urldate = {2025-01-14},
	journal = {Advances in Neural Information Processing Systems},
	author = {Shah, Kulin and Chen, Sitan and Klivans, Adam},
	month = dec,
	year = {2023},
	pages = {19636--19649},
}

@inproceedings{rombach_high-resolution_2022,
	address = {New Orleans, LA, USA},
	title = {High-{Resolution} {Image} {Synthesis} with {Latent} {Diffusion} {Models}},
	copyright = {https://doi.org/10.15223/policy-029},
	language = {en},
	urldate = {2025-01-13},
	booktitle = {2022 {IEEE}/{CVF} {Conference} on {Computer} {Vision} and {Pattern} {Recognition} ({CVPR})},
	publisher = {IEEE},
	author = {Rombach, Robin and Blattmann, Andreas and Lorenz, Dominik and Esser, Patrick and Ommer, Bjorn},
	month = jun,
	year = {2022},
	pages = {10674--10685},
}

@inproceedings{rahimi_random_2007,
	title = {Random {Features} for {Large}-{Scale} {Kernel} {Machines}},
	volume = {20},
	urldate = {2025-01-13},
	booktitle = {Advances in {Neural} {Information} {Processing} {Systems}},
	publisher = {Curran Associates, Inc.},
	author = {Rahimi, Ali and Recht, Benjamin},
	year = {2007},
}

@article{hu_universality_2023,
	title = {Universality {Laws} for {High}-{Dimensional} {Learning} {With} {Random} {Features}},
	volume = {69},
	number = {3},
	urldate = {2025-01-13},
	journal = {IEEE Transactions on Information Theory},
	author = {Hu, Hong and Lu, Yue M.},
	month = mar,
	year = {2023},
	note = {Conference Name: IEEE Transactions on Information Theory},
	pages = {1932--1964},
}

@inproceedings{ho_denoising_2020,
	title = {Denoising {Diffusion} {Probabilistic} {Models}},
	volume = {33},
	urldate = {2024-05-22},
	booktitle = {Advances in {Neural} {Information} {Processing} {Systems}},
	publisher = {Curran Associates, Inc.},
	author = {Ho, Jonathan and Jain, Ajay and Abbeel, Pieter},
	year = {2020},
	pages = {6840--6851},
}

@inproceedings{dhariwal_diffusion_2021,
	title = {Diffusion {Models} {Beat} {GANs} on {Image} {Synthesis}},
	volume = {34},
	urldate = {2025-01-13},
	booktitle = {Advances in {Neural} {Information} {Processing} {Systems}},
	publisher = {Curran Associates, Inc.},
	author = {Dhariwal, Prafulla and Nichol, Alexander},
	year = {2021},
	pages = {8780--8794},
}

@article{bortoli_convergence_2022,
	title = {Convergence of denoising diffusion models under the manifold hypothesis},
	issn = {2835-8856},
	language = {en},
	urldate = {2025-01-14},
	journal = {Transactions on Machine Learning Research},
	author = {Bortoli, Valentin De},
	month = aug,
	year = {2022},
}

@article{biroli_dynamical_2024,
	title = {Dynamical regimes of diffusion models},
	volume = {15},
	copyright = {2024 The Author(s)},
	issn = {2041-1723},
	doi = {10.1038/s41467-024-54281-3},
	language = {en},
	number = {1},
	urldate = {2025-01-08},
	journal = {Nature Communications},
	publisher = {Nature Publishing Group},
	author = {Biroli, Giulio and Bonnaire, Tony and de Bortoli, Valentin and Mézard, Marc},
	month = nov,
	year = {2024},
	pages = {9957},
}

@article{anderson_reverse-time_1982,
	title = {Reverse-time diffusion equation models},
	volume = {12},
	number = {3},
	urldate = {2025-01-10},
	journal = {Stochastic Processes and their Applications},
	author = {Anderson, Brian D. O.},
	month = may,
	year = {1982},
	pages = {313--326},
}

@inproceedings{chen_improved_2023,
	address = {Honolulu, Hawaii, USA},
	series = {{ICML}'23},
	title = {Improved analysis of score-based generative modeling: user-friendly bounds under minimal smoothness assumptions},
	volume = {202},
	shorttitle = {Improved analysis of score-based generative modeling},
	urldate = {2025-01-13},
	booktitle = {Proceedings of the 40th {International} {Conference} on {Machine} {Learning}},
	publisher = {JMLR.org},
	author = {Chen, Hongrui and Lee, Holden and Lu, Jianfeng},
	month = jul,
	year = {2023},
	pages = {4735--4763},
}

\clearpage
\appendix
\section{Learning the score function using RFNN}
Since the setting that we have is a least squared optimization problem, we can compute the optimizer analytically. However, in practice, stochastic gradient descent is used for optimization. First we derive analytical expressions for the optimal RFNN score.
\subsection{Optimal Score}\label{apndx_sec:optimal_rfnn_score}
We want to obtain the minimizer $\hat{A}_t$ of the loss function \eqref{eqn:dsm_rfnn} which we copy below. 
\begin{equation}
    \cL_t(A_t) =\frac{1}{dn}\sum_{i=1}^{n}{\shortexpect_{z}{\norm{\sqrt{h_t}\rat{A_t}{p}\act{\rat{W_t}{p}(a_tx_i+\sqrt{h_t}z)}+z}^2}}+\frac{h_t\lambda}{dp}\norm{A_t}_F^2.
\end{equation}
We have
\begin{align*}
    \cL_t(A_t) &=\frac{1}{dn}\sum_{i=1}^{n}{\shortexpect_{z}{\norm{\sqrt{h_t}\rat{A_t}{p}\act{\rat{W_t}{p}(a_tx_i+\sqrt{h_t}z)}+z}^2}}+\frac{h_t\lambda}{dp}\norm{A_t}_F^2,\\
    &=\frac{h_t}{d}\tr{\rat{A_t}{p}^T\rat{A_t}{p} U}+\frac{2\sqrt{h_t}}{d}\tr{\rat{A_t}{p} V}+1+\frac{h_t\lambda}{d}\tr{\rat{A_t^T}{p}\rat{A_t}{p}},\\
\end{align*}
where 
$$U = \frac{1}{n}\sum_{i=1}^n\bE{z}{\varrho(\rat{W_t}{d}(a_tx_i+\sqrt{h_t}z))\varrho(\rat{W_t}{d}(a_tx_i+\sqrt{h_t}z))^T},$$
and 
$$V = \frac{1}{n}\sum_{i=1}^n\bE{z}{\varrho(\rat{W_t}{d}(a_tx_i+\sqrt{h_t}z))z^T}.$$ Thus we get the optimal $A_t$ as
\begin{equation}
    \rat{\hat{A}_t}{p} = -\frac{1}{\sqrt{h_t}}V^T(U+\lambda I_p)^{-1}.
\end{equation}

\subsection{Exact Score for Linear Manifolds}\label{apndx_sec:exact_score_linear}
In the special case when data lie on a linear manifold, the expression for $\cE_{test}^\ast$ simplifies and admits a closed-form characterization.
\begin{lemma}
    When $\sigma(x)=x$, the exact score function $\nabla\log P_t$ is given by 
    \begin{equation*}
        s^\ast(x) = \nabla\log P_t(x) = -\left(a_t^2\frac{MM^T}{D}+h_tI\right)^{-1}x,
    \end{equation*}
    and the associated test error is given by
    \begin{equation}\label{eqn:test_train_linear_man}
        \lim_{d\to\infty}\bE{}{\cE_{\text{test}}^\ast} = 1-\frac{h_t}{a_t^2}s_{MP}\left(-\frac{h_t}{a_t^2},\frac{1}{\psi_D}\right),
    \end{equation}
where $s_{MP}$ is the Stieltjes transform of the Marchenko-Pastur distribution.
\end{lemma}
\begin{proof}
    We have
    \begin{align*}
        \cE_{\text{test}}^\ast &= \frac{1}{d}\shortexpect_{x,z}{\norm{\sqrt{h_t}s^\ast(a_tx+\sqrt{h_t}z)+z}^2},\\
        &=\frac{h_t}{d}\tr{\left(a_t^2\frac{MM^T}{D}+h_tI\right)^{-1}}-\frac{2h_t}{d}\tr{\left(a_t^2\frac{MM^T}{D}+h_tI\right)^{-1}}+1,\\
        &= 1-\frac{h_t}{a_t^2}\frac{1}{d}\tr{\left(\frac{MM^T}{D}+\frac{h_t}{a_t^2}I\right)^{-1}}.
    \end{align*}
The result follows by taking the limit $d\to\infty$.
\end{proof}

\section{Proof of Theorem~\ref{thm:test_train}}\label{appnxd_sec:proof_test_train}
We expect the test and train errors to concentrate around its expectations as $d$ grows. In this section, we derive their expected values. Let the activation function $\act{\cdot}$ and the manifold folding function $\man{\cdot}$ satisfy Assumption~\ref{assumption:act_man}.

\textbf{I. Test error:}
We expand the expression for test error as follows:
\begin{align*}
    \cE_{\text{test}}(\hat{A}_t) &= \frac{1}{d}\bE{x\sim P_t,z}{\norm{\sqrt{h_t}\rat{\hat{A}_t}{p}\act{\rat{W_t}{d}(a_tx+\sqrt{h_t}z)}+z}^2}\\
    &=1 - \frac{2}{d}\tr{V^T (U+\lambda I_p)^{-1}\underbrace{\bE{x,z}{\act{\rat{W_t}{d}(a_tx+\sqrt{h_t}z)}z^T}}_{:=\Tilde{V}}}\\
    &\qquad+\frac{1}{d}\tr{(U+\lambda I_p)^{-1}V V^T(U+\lambda I_p)^{-1}\underbrace{\bE{x,z}{\act{\rat{W_t}{d}(a_tx+\sqrt{h_t}z)}\act{\rat{W_t}{d}(a_tx+\sqrt{h_t}z)}^T}}_{:=\Tilde{U}}}\;.
\end{align*}
Since we focus on a single time instant, we drop the subscript $t$ in the above expressions. However, it is important to note that $a$ and $h$ depend on $t$, and we have the relation $a^2+h=1$.

We need to compute $V,U,\Tilde{V},\Tilde{U}$ in order to get an expression for $\cE_{\text{test}}$. Note that, in order to derive an asymptotically precise expression for $\cE_{\text{test}}$, it suffices to obtain $V,U,\Tilde{V},\Tilde{U}$ to leading order. In particular, we can neglect $O(1/d^{1-\epsilon})$ terms entry-wise in these matrices, for $0\le\epsilon<1/2$ and for the diagonal terms $\epsilon$ can be in $[0,1)$. This can be justifies as follows. Using Claim 1 (see Lemma~\ref{lemma:gep_sec_moment} in Appendix~\ref{apndx_sec:proof_lemma_gep}) the error matrix $\Delta$ has, with high probability, entries  $O(1/d^{1-\epsilon})$, where here high probability means a probability tending to $1$ faster than any inverse power. Consequently, by the union bound, the error matrix will have Frobenius $\norm{\Delta}_F = O(d^\epsilon)$ with probability tending to $1$ faster than any power. In the expressions for test and train errors, these error matrices typically appears as $\frac{1}{d}\tr{\Delta A}$, where $A$ is a matrix with operator norm $\bigO{1}$ w.r.t. $d$. Using standard inequalities, we can upper bound $\frac{1}{d}\tr{\Delta A}$ as
\begin{align*}
    \frac{1}{d}\tr{\Delta A} &\le\frac{1}{d}\norm{\Delta}_\ast\norm{A}_{\text{op}}\\
    &\le \frac{1}{d}\sqrt{\text{rank}(\Delta)}\norm{\Delta}_F\norm{A}_{\text{op}}\\
    &= \bigO{\frac{1}{\sqrt{d}}}\norm{\Delta}_F = O(d^{\epsilon-1/2}). 
\end{align*}
where the last equality holds with probability tending to $1$ faster than any inverse power. 
Thus, asymptotically these error terms vanish with high probability. 

\textbf{(i) We will first consider $\Tilde{V}$:} We have
\begin{equation*}
    \Tilde{V} = \bE{x,z}{\act{\rat{W}{d}(ax+\sqrt{h}z)}z^T}\;.
\end{equation*}
Let $P^\gamma$ denote the bivariate standard Gaussian distribution with correlation coefficient $\gamma$. Explicitly,
\begin{equation}\label{eqn:bivariate_gaussian_pdf}
    P^\gamma(x,y) = \frac{1}{2\pi\sqrt{1-\gamma^2}}e^{-\frac{x^2+y^2-2\gamma xy}{2(1-\gamma^2)}}\;.
\end{equation}
We also recall the Mehler kernel formula \cite{kibble_extension_1945}, which will be extensively used in the proof. Let $f,g:\R\to\R$ be a functions that are square integrable w.r.t. Gaussian measure. Let $\text{He}_k$ be the $k^{th}$ probablist's Hermite polynomial. Then, for $(u,v)\sim P^\gamma$, the Mehler kernel formula gives
\begin{equation}\label{eqn:Mehler_kernel}
    \bE{}{f(u)g(v)} = \sum_{k=0}^\infty \frac{\gamma^k}{k!}\bE{}{f(u)\text{He}_k(u)}\bE{}{g(v)\text{He}_k(v)}.
\end{equation}
Let $w_i$ denote the $i^{\text{th}}$ row of $W$. For large $d$, $\frac{\norm{w_i}^2}{d}$ concentrates to $1$. Let $x'\sim\cN{0,\cC}$, where $\cC = \nu_1^2\frac{MM^T}{d}+\sqrt{1-\nu^2}I$. Then:
\begin{align*}
        \Tilde{V}_{ij} &= \bE{x,z}{\act{\rat{w_i^T(ax+\sqrt{h}z)}{d}}z_j}\;,\\
    &\stackrel{(a)}{=}\rat{w_{ij}}{d}\bE{x,u}{\act{a\rat{w_i^Tx}{d}+\sqrt{h}u}u}\;,\\
    &\stackrel{(b)}{=}\rat{w_{ij}}{d}\bE{x',u}{\act{a\rat{w_i^Tx'}{d}+\sqrt{h}u}u} + O(1/d)\;,\\
    &\stackrel{(c)}{=}\sqrt{h}\mu_1\rat{w_{ij}}{d} +O(1/d)\;,
\end{align*}
where in $(a)$ we used the Mehler Kernel formula \eqref{eqn:Mehler_kernel}, in $(b)$ we used the CLT statement for $\rat{w_i}{d}x$, stated in the proof of Lemma~\ref{lemma:gep_sec_moment} (see Appendix~\ref{apndx_sec:proof_lemma_gep}), and in $(c)$ we used the fact that $\rat{w_i^T}{d}x'\sim\cN{0,1}$ asymptotically.
Hence, we have $\Tilde{V} = \sqrt{h}\mu_1\rat{W}{d}$. 

 \textbf{(ii) Now, we consider $\Tilde{U}$:} The matrix elements are
\begin{align*}
    \Tilde{U}_{ij} &= \bE{x,z}{\act{\rat{w_i^T(ax+\sqrt{h}z)}{d}}\act{\rat{w_j^T(ax+\sqrt{h}z)}{d}}},\\
    &\stackrel{(a)}{=} \bE{x',z}{\act{\rat{w_i^T(ax'+\sqrt{h}z)}{d}}\act{\rat{w_j^T(ax'+\sqrt{h}z)}{d}}}+O(1/d^{1-\epsilon}),\\
    &\stackrel{(a)}{=} \mu_1^2\rat{w_i^T}{d}(a^2\cC+h I_d)\rat{w_j}{d}+O(1/d^{1-\epsilon})\;.
\end{align*}
We used the result from Claim~\ref{lemma:gep_sec_moment} in $(a)$ and Mehler kernel formula in $(b)$. 
Therefore,
\begin{equation*}
    \Tilde{U}_{ij} = \begin{cases}
        \mu_1^2\rat{w_i^T}{d}(a^2\cC+h I_d)\rat{w_j}{d} + O(1/d^{1-\epsilon})\quad &\text{if } i\neq j \;,\\
        1\quad &\text{if } i=j \;,
    \end{cases}
\end{equation*}
which gives
\begin{equation*}
    \Tilde{U} = \mu_1^2 \rat{W}{d}(a^2\cC+h I_d)\rat{W^T}{d} + (1-\mu_1^2) I_p \;.
\end{equation*}
\textbf{(iii)
Now we will consider $V$:}
Let
\begin{equation*}
    V^l = \bE{z}{\act{\rat{W}{d}(ax_l+\sqrt{h}z)}z^T}\;.
\end{equation*}
We have again by using Mehler's kernel formula,
\begin{align*}
    V^l_{ij} &= \bE{z}{\act{\rat{w_i^T(ax_l+\sqrt{h}z)}{d}}z_j},\\
    &= \bE{(u,v)\sim P^{\rat{w_{ij}}{d}}}{\act{\rat{a w_i^Tx_l}{d}+\sqrt{h}u}v},\\
    &= \sum_{k=0}^{\infty} \frac{(\rat{w_{ij}}{d})^k}{k!}\bE{u}{\act{\rat{a w_i^Tx_l}{d}+\sqrt{h}u}\text{He}_k(u)}\bE{v}{v\text{He}_k(v)},\\
    &= \rat{w_{ij}}{d}\bE{u}{\act{\rat{a w_i^Tx_l}{d}+\sqrt{h}u}u} + O(1/d),\\
    &= \rat{w_{ij}}{d}\varrho_1\left(\rat{w_i^Tx_l}{d}\right) + O(1/d)\;,
\end{align*}
where $\varrho_1(y) = \bE{u}{\varrho(ay+\sqrt{h}u)u}$. Summing over the $n$ data samples:
\begin{align*}
    V_{ij} &= \frac{1}{n}\sum_{l=1}^{n} V^l_{ij},\\
    &= \rat{w_{ij}}{d} \frac{1}{n}\sum_{l=1}^{n} \varrho_1\left(\rat{ w_i^Tx_l}{d}\right),\\
    &= \rat{w_{ij}}{d} \bE{x}{\varrho_1\left(\rat{w_i^Tx}{d}\right)} + O(1/d),\\
    &= \rat{w_{ij}}{d} \bE{g}{\varrho_1(g)} + O(1/d),\\
    &= \rat{w_{ij}}{d} \bE{g,u}{\varrho(ag+\sqrt{h}u)u} + O(1/d),\\
    &= \sqrt{h}\mu_1\rat{w_{ij}}{d} + O(1/d)\;.
\end{align*}
Neglecting $O(1/d)$ terms, we have $V = \sqrt{h}\mu_1\rat{W}{d}$. 

\textbf{Now, let's consider $U$:} Let
\begin{align*}
    U^l &= \bE{z}{\act{\rat{W}{d}(ax_l+\sqrt{h}z)}\act{\rat{W}{d}(ax_l+\sqrt{h}z)}^T}\;.
\end{align*}
For $i\neq j $ we have,
\begin{align*}
    U^l_{ij} &= \bE{z}{\act{\rat{w_i^T(ax_l+\sqrt{h}z)}{d}}\act{\rat{w_j^T(ax_l+\sqrt{h}z)}{d}}}\\
    &= \bE{(u,v)\sim P^{\frac{w_i^Tw_j}{d}}}{\act{a\rat{w_i^Tx_l}{d} + \sqrt{h}u}\act{a\rat{w_j^Tx_l}{d} + \sqrt{h}v}}\\
    &= \sum_{k=0}^{\infty} \frac{(\frac{w_i^Tw_j}{d})^k}{k!}\bE{u}{\act{a\rat{w_i^Tx_l}{d} + \sqrt{h}u}\text{He}_k(u)}\bE{v}{\act{a\rat{w_j^Tx_l}{d} + \sqrt{h}v}\text{He}_k(v)}\\
    &= \varrho_0\left(\rat{w_i^Tx_l}{d}\right)\varrho_0\left(\rat{w_j^Tx_l}{d}\right)+\frac{w_i^Tw_j}{d}\varrho_1\left(\rat{w_i^Tx_l}{d}\right)\varrho_1\left(\rat{w_j^Tx_l}{d}\right)+O(1/d)\;,
\end{align*}
where $\varrho_0(y) = \bE{u}{\varrho(ay+\sqrt{h}u)}$ and $\varrho_1(y) = \bE{u}{\varrho(ay+\sqrt{h}u)u}$. Summing over the $n$ data samples:
\begin{align*}
    U_{ij} &= \frac{1}{n}\sum_{l=1}^{n} U^l_{ij}\\
    &= \frac{1}{n}\sum_{l=1}^{n}\varrho_0\left(\rat{w_i^Tx_l}{d}\right)\varrho_0\left(\rat{w_j^Tx_l}{d}\right)+\frac{w_i^Tw_j}{d}\frac{1}{n}\sum_{l=1}^{n}\varrho_1\left(\rat{w_i^Tx_l}{d}\right)\varrho_1\left(\rat{w_j^Tx_l}{d}\right) + O(1/d)\\
    &= \frac{1}{n}\sum_{l=1}^{n}\varrho_0\left(\rat{w_i^Tx_l}{d}\right)\varrho_0\left(\rat{w_j^Tx_l}{d}\right)+\frac{w_i^Tw_j}{d}\bE{x}{\varrho_1\left(\rat{w_i^Tx}{d}\right)\varrho_1\left(\rat{w_j^Tx}{d}\right)} + O(1/d)\\
    &\stackrel{(a)}{=} \frac{1}{n}\sum_{l=1}^{n}\varrho_0\left(\rat{w_i^Tx_l}{d}\right)\varrho_0\left(\rat{w_j^Tx_l}{d}\right)+\frac{w_i^Tw_j}{d}\bE{g}{\varrho_1(g)}^2 + O(1/d)\\
    &= \frac{1}{n}\sum_{l=1}^{n}\varrho_0\left(\rat{w_i^Tx_l}{d}\right)\varrho_0\left(\rat{w_j^Tx_l}{d}\right)+h\mu_1^2\frac{w_i^Tw_j}{d} + O(1/d)\;.\\
\end{align*}
In $(a)$ we the used Mehler kernel formula on $\bE{x}{\varrho_1\left(\rat{w_i^Tx}{d}\right)\varrho_1\left(\rat{w_j^Tx}{d}\right)}$ and neglected terms that will lead to $O(1/d)$ terms.
For $i=j$, we have:
\begin{align*}
    U^l_{ii} &= \bE{z}{\left(\act{\rat{w_i^T(ax_l+\sqrt{h}z)}{d}}\right)^2} \;,\\
\end{align*}
and
\begin{align*}
    U_{ii} &= \frac{1}{n}\sum_{l=1}^n\bE{z}{\left(\act{\rat{w_i^T(ax_l+\sqrt{h}z)}{d}}\right)^2}\\
    &= \bE{z,x}{\left(\act{\rat{w_i^T(ax_l+\sqrt{h}z)}{d}}\right)^2}+O(1/\sqrt{d})\\
    &= \norm{\varrho}^2+O(1/\sqrt{d}) \;.
\end{align*}
The $\bigO{1/\sqrt{d}}$ term in the above equation can be neglected, since there are only $\bigO{d}$ terms on the diagonal. Let $X=[x_1,x_2,\cdots,x_n]\in\R^{d\times n}$. We can write $U$ as:
\begin{equation*}
    U = \rat{\varrho_0\left(\rat{W}{d}X\right)}{n}\rat{\varrho_0\left(\rat{W}{d}X\right)^T}{n}+h\mu_1^2\rat{W}{d}\rat{W^T}{d}+s^2I_p \;,
\end{equation*}
where
\begin{align*}
s^2 &= \norm{\varrho}^2-\bE{g}{\varrho_0(g)^2}-h\mu_1^2\\
&= \norm{\varrho}^2-\bE{g}{\bE{u}{\varrho(ag+\sqrt{h}u)}^2}-h\mu_1^2\\
&= \norm{\varrho}^2-c(a^2)-h\mu_1^2 \;,\\
&= 1-c(a^2)-h\mu_1^2 \;,
\end{align*}
with $c(\gamma) = \bE{u,v\sim P^\gamma}{\varrho(u)\varrho(v)}$. Let $F:= \varrho_0\left(\rat{W}{d}X\right)$.
We have found:
\begin{equation*}
    U = \rat{F}{n}\rat{F^T}{n}+h\mu_1^2\rat{W}{d}\rat{W^T}{d}+s^2I_p \;,
\end{equation*}

\textbf{(v) Final derivations for test error:} We now have all the building blocks to compute the test error.
\begin{align*}
    \cE_{\text{test}}(\hat{A}_t) &= 1 - \frac{2h\mu_1^2}{d}\tr{\rat{W^T}{d} (U+\lambda I_p)^{-1}\rat{W}{d}}\\
    &\qquad+\frac{h\mu_1^2}{d}\tr{(U+\lambda I_p)^{-1}\rat{W}{d}\rat{W^T}{d}(U+\lambda I_p)^{-1}\left(\mu_1^2 \rat{W}{d}\underbrace{(a^2\cC+h I_d)}_{:=\Sigma}\rat{W^T}{d} + (1-\mu_1^2) I_p\right)}\;,\\
    &= 1-2h\mu_1^2E_1+h\mu_1^4E_2+h\mu_1^2(1-\mu_1^2)E_3,
\end{align*}
where 
\begin{align*}
    E_1 &= \frac{1}{d}\tr{\rat{W^T}{d} (U+\lambda I_p)^{-1}\rat{W}{d}}\\
    E_2 &= \frac{1}{d}\tr{\rat{W^T}{d}(U+\lambda I_p)^{-1}\rat{W}{d}\Sigma\rat{W^T}{d}(U+\lambda I_p)^{-1}\rat{W}{d}}\\
    E_3 &= \frac{1}{d}\tr{\rat{W^T}{d} (U+\lambda I_p)^{-2}\rat{W}{d}}.
\end{align*}
By taking expectation and large $d$ limit, we have
\begin{align*}
    \lim_{d\to\infty}\bE{}{\cE_{\text{test}}(\hat{A}_t)} &= 1-2h\mu_1^2e_1+h\mu_1^4e_2+h\mu_1^2(1-\mu_1^2)e_3,
\end{align*}
where $e_i = \lim_{d\to\infty} \bE{}{E_i}$ for $i=1,2,3$.
Now define the following {\it new} matrix
\begin{equation*}
    U(q) :=  \rat{F}{n}\rat{F^T}{n}+h\mu_1^2\rat{W}{d}\rat{W^T}{d}+q\rat{W}{d}\Sigma\rat{W^T}{d}+s^2I_p \;,
\end{equation*}
and the resolvent of $U(q)$ as
\begin{equation}\label{apndx_eqnl:resolvent}
    R(q,z) = (U(q)-zI_p)^{-1}\;.
\end{equation}
Let 
\begin{equation*}
    K(q,z) = \frac{1}{d}\tr{\rat{W^T}{d}R(q,z)\rat{W}{d}},\quad\text{and}\quad \cK(q,z) = \lim_{d\to\infty}\frac{1}{d}\tr{\rat{W^T}{d}R(q,z)\rat{W}{d}}\;.
\end{equation*}
Using the identities $\frac{\partial R}{\partial q} = -R(q,z)\frac{\dd U}{\dd q}R(q,z)$ and $\frac{\partial R}{\partial z} = R(q,z)^2$, we observe that 
\begin{align*}
    E_1 &= K(0,-\lambda)\;,\\
    E_2 &= -\frac{\partial K}{\partial q}(0,-\lambda)\;,\\
    E_3 &= \frac{\partial K}{\partial z}(0,-\lambda)\;.
\end{align*}
Since we want $\bE{W,M,\Xi}{\cE_{\text{test}}(\hat{A}_t)}$, it suffices to compute $\bE{M,W,\Xi}{K(q,z)}$, where $\Xi=[\xi_1,\xi_2,\cdots,\xi_n]$.
\begin{equation*}
    \bE{M,W,\Xi}{K(q,z)} =  \frac{1}{d}\tr{\bE{M,W}{\rat{W^T}{d}\bE{\Xi}{R(q,z)}\rat{W}{d}}}\;.
\end{equation*}
The exchange of expectation and derivative is justified by the bounded operator norm of $R(q,z)$ for a complex $z$ with positive imaginary part.

To compute $\bE{M,W,\Xi}{K(q,z)}$, we derive a set of self consistent equations satisfied by it. Lemma~\ref{lemma:det_equiv} is the main result that we use to derive the self-consistent equations. Lemma~\ref{lemma:det_equiv} follows from Theorem~2.18 in \cite{couillet_random_2022}.
\begin{lemma}\label{lemma:det_equiv}
    Let $\Phi=[\phi_1,\phi_2,\cdots,\phi_n]\in\R^{p\times n}$ be a random matrix with $\phi_i\stackrel{i.i.d.}{\sim}P_\phi$, and $\phi_i$'s are concentrating random vectors. Let $\Sigma_\phi = \bE{}{\phi\phi^T}$, and let $\Theta_1,\Theta_2$ be constant p.s.d. matrices with bounded operator norm. Then, for $\kappa>0$
    \begin{equation}
        \bE{}{\left(\frac{1}{n}\Theta_1\Phi\Phi^T+\Theta_2+\kappa I\right)^{-1}} = \left(\frac{1}{1+\zeta_n}\Theta_1\Sigma_\phi+\Theta_2+\kappa I\right)^{-1}+\Tilde{\Delta},
    \end{equation}
    where $\norm{\Tilde{\Delta}}_{op} = o(1)$ and $\zeta_n$ is a the solution to the self-consistent equation 
    \begin{equation}\label{eqn:self_consistent}
        \zeta_n = \frac{1}{n}\tr{\left(\frac{1}{1+\zeta_n}\Theta_1\Sigma_\phi+\Theta_2+\kappa I\right)^{-1}\Theta_1\Sigma_\phi}.
    \end{equation}
\end{lemma}

First we use Lemma~\ref{lemma:det_equiv} to compute $\bE{\Xi}{R(q,z)}$. Let $f_i = \varrho_0(\rat{W}{d}x_i)$, and $F=[f_1,f_2,\cdots,f_n]$. The Lipschitz assumption on $\varrho$ and $\sigma$ allows us to apply  Lemma~\ref{lemma:det_equiv} here. We get
\begin{align}\label{eq:det_equiv1}
    \bE{\Xi}{R(q,z)} &=  \bE{\Xi}{\left(\rat{F}{n}\rat{F^T}{n}+h\mu_1^2\rat{W}{d}\rat{W^T}{d}+q\rat{W}{d}\Sigma\rat{W^T}{d}+(s^2-z)I_p\right)^{-1}}\;,\nonumber\\
    &= \left(\frac{1}{1+\frac{1}{n}\tr{\bE{}{R}\Sigma_f}}\Sigma_f+h\mu_1^2\rat{W}{d}\rat{W^T}{d}+q\rat{W}{d}\Sigma\rat{W^T}{d}+(s^2-z)I_p\right)^{-1}\;,\\
\end{align}
where using \eqref{eqn:gep_sec_moment} and Mehler's kernel formula, we have
\begin{align*}
    \Sigma_f = \bE{x}{\varrho_0\left(\rat{W}{d}x\right)\varrho_0\left(\rat{W}{d}x\right)^T} &= a^2\mu_1^2\rat{W}{d}\cC\rat{W^T}{d}+(c(a^2)-a^2\mu_1^2)I_p + O(1/d^{1-\epsilon}). 
\end{align*}
Let 
\begin{align*}
    \zeta_1 &= \frac{1}{n}\tr{\bE{}{R}\rat{W}{d}\cC\rat{W^T}{d}},\\
    \zeta_2 &= \frac{1}{n}\tr{\bE{}{R}}.
\end{align*}
Substituting the expression for $\Sigma_f$ and $\Sigma = a^2\cC+hI_d$ in \eqref{eq:det_equiv1}, we get
\begin{equation}\label{eqn:expect_Chi_R}
    \bE{\Xi}{R(q,z)} =  {\left(\rat{W}{d}\underbrace{(\kappa_1\cC+\kappa_2I)}_{:=\Theta}\rat{W^T}{d}+\kappa_3I_p\right)^{-1}}\;,
\end{equation}
where
\begin{align*}
    \kappa_1 &= \frac{a^2\mu_1^2}{\underbrace{1+a^2\mu_1^2\zeta_1+\underbrace{(c(a^2)-a^2\mu_1^2)}_{:=v^2}\zeta_2}_{:=\chi(\zeta_1,\zeta_2)}}+qa^2 = a^2\left(\frac{\mu_1^2}{\chi(\zeta_1,\zeta_2)}+q\right)\\ 
    \kappa_2 &= h(\mu_1^2+q),\\
    \kappa_3 &= s^2-z+\frac{v^2}{\chi(\zeta_1,\zeta_2)}.
\end{align*}
Towards computing $K(q,z)$, next we evaluate $\bE{W}{\rat{W^T}{d}\bE{\Xi}{R}\rat{W}{d}}$.
\begin{align*}
    \bE{W}{\rat{W^T}{d}\bE{\Xi}{R}\rat{W}{d}} &= \bE{W}{\rat{W^T}{d}\rat{W}{d}\left(\Theta\rat{W^T}{d}\rat{W}{d}+\kappa_3 I\right)^{-1}},\\
    &= \Theta^{-1}\bE{W}{\Theta\rat{W^T}{d}\rat{W}{d}\left(\Theta\rat{W^T}{d}\rat{W}{d}+\kappa_3 I\right)^{-1}},\\
    &= \Theta^{-1}-\kappa_3\Theta^{-1}\bE{W}{\left(\psi_p\frac{1}{p}\Theta W^TW+\kappa_3 I\right)^{-1}}.
\end{align*}
Once again using Lemma~\ref{lemma:det_equiv}, we get
\begin{equation*}
    \bE{W}{\left(\psi_p\frac{1}{p}\Theta W^TW+\kappa_3 I\right)^{-1}} = \left(\frac{\psi_p}{1+\zeta_3}\Theta+\kappa_3 I\right)^{-1},
\end{equation*}
where
\begin{equation}\label{eqn:zeta3_def}
    \zeta_3 = \frac{1}{d}\tr{\left(\frac{\psi_p}{1+\zeta_3}\Theta+\kappa_3\right)^{-1}\Theta}.
\end{equation}
Hence,
\begin{align}
    \bE{W}{\rat{W^T}{d}\bE{\Xi}{R}\rat{W}{d}} &=  \Theta^{-1}\left\{I-\kappa_3\left(\frac{\psi_p}{1+\zeta_3}\Theta+\kappa_3 I\right)^{-1}\right\}\nonumber\\
    &= \frac{\psi_p}{1+\zeta_3}\left(\frac{\psi_p}{1+\zeta_3}\Theta+\kappa_3 I\right)^{-1}\label{eqn:WRW_in_zeta3}\\
    &= \frac{\psi_p}{1+\zeta_3}\left(\frac{\psi_p}{1+\zeta_3}(\kappa_1\cC+\kappa_2I)+\kappa_3 I\right)^{-1}\nonumber\\
    &= \frac{\psi_p}{1+\zeta_3}\left(\frac{\psi_p}{1+\zeta_3}\left(\kappa_1\left(\nu_1^2 \frac{MM^T}{D}+(1-\nu_1^2)I_d\right)+\kappa_2I\right)+\kappa_3 I\right)^{-1}\nonumber\\
    &= \kappa_4\left(\frac{MM^T}{D}+\kappa_5 I\right)^{-1},\label{eqn:det_eqiv_WRW}
\end{align}
where
\begin{align*}
    \kappa_4 &= \frac{1}{\kappa_1\nu_1^2},\\
    \kappa_5 &= \frac{\kappa_1(1-\nu_1^2)+\kappa_2+\kappa_3\frac{1+\zeta_3}{\psi_p}}{\kappa_1\nu_1^2}.
\end{align*}
By again applying Lemma~\ref{lemma:det_equiv},
\begin{align}
    \bE{M,W}{\rat{W^T}{d}\bE{\Xi}{R}\rat{W}{d}} &= \kappa_4\bE{M}{\left(\frac{MM^T}{D}+\kappa_5 I\right)^{-1}},\nonumber\\
    &= \kappa_4 \left(\frac{1}{1+\zeta_4}+\kappa_5\right)^{-1}I_d,\label{eqn:expect_M_WRW}
\end{align}
where
\begin{equation}\label{eqn:zeta4_def}
    \zeta_4 = \frac{1}{D}\tr{\bE{M}{\left(\frac{MM^T}{D}+\kappa_5 I\right)^{-1}}}.
\end{equation}
Therefore,
\begin{equation}\label{eqn:K_in_zeta}
    K(q,z) = \frac{1}{d}\tr{\rat{W^T}{d}R(q,z)\rat{W}{d}}=\kappa_4\psi_D\zeta_4.
\end{equation}
Now we can close the equations and obtain self-consistent equations for $\zeta_1,\zeta_2,\zeta_3,\zeta_4$. We have
\begin{align*}
    \zeta_1 &= \frac{1}{n}\tr{\bE{}{R}\rat{W}{d}\cC\rat{W^T}{d}},\\
    &= \frac{1}{n}\tr{\rat{W^T}{d}R\rat{W}{d}\cC},\\
    &\stackrel{(a)}{=} \frac{1}{n}\tr{\kappa_4\left(\frac{MM^T}{D}+\kappa_5 I\right)^{-1}\left(\nu_1^2 \frac{MM^T}{D}+(1-\nu_1^2)I_d\right)},\\
    &= \frac{1}{n}\tr{\kappa_4\nu_1^2\left\{I-\kappa_5\left(\frac{MM^T}{D}+\kappa_5 I\right)^{-1}\right\}+\kappa_4(1-\nu_1^2)\left(\frac{MM^T}{D}+\kappa_5 I\right)^{-1}},\\
    &\stackrel{(b)}{=} \frac{\kappa_4\nu_1^2}{\psi_n}-\frac{\psi_D}{\psi_n}\kappa_4\kappa_5\nu_1^2\zeta_4+\frac{\psi_D}{\psi_n}\kappa_4(1-\nu_1^2)\zeta_4,
\end{align*}
where in $(a)$ we used \eqref{eqn:det_eqiv_WRW}, and in $(b)$ we used \eqref{eqn:zeta4_def}. Substituting for $\kappa_4,\kappa_5$ gives
\begin{equation}
    \psi_n a^4\left(\frac{\mu^2}{\chi} + q\right)^2 \zeta_1+\frac{\psi_D h(\mu^2+q)\zeta_4}{\nu_1^2}+\frac{\psi_D\left(s^2 + \frac{v^2}{\chi} - z\right)\zeta_4 (1+\zeta_3)}{\psi_p\nu_1^2}-a^2\left(\frac{\mu^2}{\chi} + q\right) = 0.
\end{equation}
Next, we consider the definition of $\zeta_3$ given in \eqref{eqn:zeta3_def}:
\begin{align}
    \zeta_3 &= \frac{1}{d}\tr{\left(\frac{\psi_p}{1+\zeta_3}\Theta+\kappa_3\right)^{-1}\Theta},\nonumber\\
     &= \frac{1+\zeta_3}{\psi_p}\frac{1}{d}\tr{I-\kappa_3\left(\frac{\psi_p}{1+\zeta_3}\Theta+\kappa_3\right)^{-1}},\nonumber\\
    &= \frac{1+\zeta_3}{\psi_p}-\left(\frac{1+\zeta_3}{\psi_p}\right)^2\psi_D\kappa_3\kappa_4\zeta_4,
\end{align}
where in the last equation we used \eqref{eqn:WRW_in_zeta3} and \eqref{eqn:K_in_zeta}. Thus, 
\begin{equation}
    \frac{1+\zeta_3}{\psi_p}-
\frac{(s^2 + \frac{v^2}{\chi} - z)}{\nu_1^2\, a^2 \left(\frac{\mu^2}{\chi} + q\right)}\frac{(1+\zeta_3)^2 \, \psi_D \, \zeta_4}{\psi_p^{\,2}}
\;-\;
\zeta_3
= 0.
\end{equation}
From \eqref{eqn:expect_Chi_R}, we also have
\begin{align*}
    \kappa_1\bE{}{R\rat{W}{d}\cC\rat{W^T}{d}}+\kappa_2\bE{}{R\rat{W}{d}\rat{W^T}{d}}+\kappa_3\bE{}{R} &= I_p.
\end{align*}
Taking $\frac{1}{p}\tr{\cdot}$ gives
\begin{align*}
    \kappa_1\frac{\psi_n}{\psi_p}\zeta_1+\kappa_2\kappa_4\frac{\psi_D}{\psi_p}\zeta_4+\kappa_3\frac{\psi_n}{\psi_p}\zeta_2-1 &= 0.
\end{align*}
Thus we get the next equation in the set of self-consistent equation:
\begin{equation}
    a^2\left(\frac{\mu^2}{\chi} + q\right)\psi_n \zeta_1+\frac{(h\mu^2+q)\,\psi_D \zeta_4}{\nu_1^2 a^2\left(\frac{\mu^2}{\chi}+q\right)}+\left(s^2 + \frac{v^2}{\chi} - z\right)\psi_n \zeta_2-\psi_p = 0.
\end{equation}
Lastly, using \eqref{eqn:expect_M_WRW} and \eqref{eqn:zeta4_def} we have
\begin{equation*}
    \zeta_4 = \frac{1}{\psi_D\left(\frac{1}{1+\zeta_4}+\kappa_5\right)},
\end{equation*}
giving 
\begin{equation}
    \frac{\psi_D \zeta_4}{1 + \zeta_4} + \psi_D \zeta_4 \left( \frac{1 - \nu_1^2}{\nu_1^2} + \frac{h (\mu^2 + q)}{a^2 \left( \frac{\mu^2}{\chi} + q \right) \nu_1^2} + \frac{\left( s^2 + \frac{v^2}{\chi} - z \right)}{a^2 \left( \frac{\mu^2}{\chi} + q \right)} \frac{(1 + \zeta_3)}{\psi_p \nu_1^2} \right) - 1 = 0.
\end{equation}

Thus, the set of self consistent equations are given by
\begin{align*}
    \psi_n a^4\left(\frac{\mu^2}{\chi} + q\right)^2 \zeta_1+\frac{\psi_D h(\mu^2+q)\zeta_4}{\nu_1^2}+\frac{\psi_D\left(s^2 + \frac{v^2}{\chi} - z\right)\zeta_4 (1+\zeta_3)}{\psi_p\nu_1^2}-a^2\left(\frac{\mu^2}{\chi} + q\right) &= 0\\
    \frac{1+\zeta_3}{\psi_p}-
\frac{(s^2 + \frac{v^2}{\chi} - z)}{\nu_1^2\, a^2 \left(\frac{\mu^2}{\chi} + q\right)}\frac{(1+\zeta_3)^2 \, \psi_D \, \zeta_4}{\psi_p^{\,2}}-\zeta_3 &= 0,\\
a^2\left(\frac{\mu^2}{\chi} + q\right)\psi_n \zeta_1+\frac{(h\mu^2+q)\,\psi_D \zeta_4}{\nu_1^2 a^2\left(\frac{\mu^2}{\chi}+q\right)}+\left(s^2 + \frac{v^2}{\chi} - z\right)\psi_n \zeta_2-\psi_p &= 0\\
\frac{\psi_D \zeta_4}{1 + \zeta_4} + \psi_D \zeta_4 \left( \frac{1 - \nu_1^2}{\nu_1^2} + \frac{h (\mu^2 + q)}{a^2 \left( \frac{\mu^2}{\chi} + q \right) \nu_1^2} + \frac{\left( s^2 + \frac{v^2}{\chi} - z \right)}{a^2 \left( \frac{\mu^2}{\chi} + q \right)} \frac{(1 + \zeta_3)}{\psi_p \nu_1^2} \right) - 1 &= 0.
\end{align*}
Note that these equations are exact only as $d\to\infty$.\\
\textbf{II. Train error:} We now show how the train error can be computed by leveraging on the previous results.
\begin{align*}
    \cE_{\text{train}}(\hat{A}_t) &= \cL(\hat{A})-\frac{h_t\lambda}{pd}\norm{\hat{A}}^2_F\\
    &=\frac{h}{d}\tr{\rat{\hat{A}_t}{p}^T\rat{\hat{A}_t}{p} (U+\lambda I_p)}+\frac{2\sqrt{h}}{d}\tr{\rat{\hat{A}_t}{p} V}+1-\frac{h\lambda}{d}\tr{\rat{\hat{A}_t}{p}^T\rat{\hat{A}_t}{p}}\\
    &=\frac{1}{d}\tr{(U+\lambda I_p)^{-1}VV^T}-\frac{2}{d}\tr{V^T(U+\lambda I_p)^{-1}V}+1-\frac{\lambda}{d}\tr{V^T(U+\lambda I_p)^{-2}V}\\
    &=-\frac{h\mu_1^2}{d}\tr{(U+\lambda I_p)^{-1}\rat{W}{d}\rat{W^T}{d}}+1-\frac{h\mu_1^2\lambda}{d}\tr{(U+\lambda I_p)^{-2}\rat{W}{d}\rat{W^T}{d}}\\
    &=-h\mu_1^2K(0,-\lambda)-h\lambda\mu_1^2\frac{\partial K}{\partial z}(0,-\lambda)+1 \;.\\
\end{align*}
Thus,
\begin{align}\label{eqn:Etrain_minf}
    \lim_{d\to\infty}\bE{}{\cE_{\text{train}}(\hat{A}_t)} &= -h\mu_1^2\cK(0,-\lambda)-\lambda h\mu_1^2\frac{\partial \cK}{\partial z}(0,-\lambda)+1 \;,\nonumber\\
    &= 1-h\mu_1^2e_1-h\lambda\mu_1^2e_3 \;,
\end{align}
where $e_1 = \cK(0,-\lambda),\; e_3 = \frac{\partial \cK}{\partial z}(0,-\lambda)$.

\clearpage 
\section{Other proofs}\label{appsec:proofs}
\subsection{Proof of Claim~\ref{lemma:gep_sec_moment} in a restricted setting}\label{apndx_sec:proof_lemma_gep}
\begin{lemma}
    Let $f:\R\to\R$ be any polynomial function such that $\bE{g\sim\cN{0,1}}{f(g)} = 0$. Let $\sigma$ be an activation function with a finite number of terms in its Hermite expansion. Let $\phi_i = \frac{w_i^T}{\sqrt{d}}\man{\rat{M}{D}\xi}$ and $\phi_i' = \frac{w_i^T}{\sqrt{d}}\left(\nu_1\rat{M}{D}\xi'+\sqrt{1-\nu_1^2}z\right)$ for $i=1,2$, where $w_1,w_2,z\stackrel{i.i.d.}{\sim}\cN{0,I_d}$ and $\xi,\xi'\stackrel{i.i.d.}{\sim}\cN{0,I_D}$ are independent rvs. Then, one can find $\epsilon_0>0$ small enough, such that for any $0<\epsilon_0\leq \epsilon<1/2$,
    \begin{align}
        |\bE{\xi}{f(\phi_1)f(\phi_2)}-\bE{\xi',z}{f(\phi_1')f(\phi_2')}| &= O(1/d^{1-\epsilon}), \quad \text{w.h.p.}
    \end{align}
    Here, w.h.p means with probability at least $1-O\biggl(\exp(- c_{f, \sigma} \, d^{a_f \, \epsilon})\biggr)$ for some constants $c_{f, \sigma} >0$, $a_f>0$ depending only on $f$ and $\sigma$ w.r.t. $w_1, w_2$.
\end{lemma}
\begin{proof}
    Prior works \citet{goldt_modeling_2020, hu_universality_2023} have proved CLT for $(\phi_1,\phi_2)$. In particular, it follows that for any function $g$,
    \begin{equation*}
        |\bE{\xi}{g(\phi_1,\phi_2)}-\bE{\xi',z}{g(\phi'_1,\phi'_2)}| = O(1/\sqrt{d}), \quad \text{w.h.p.}
    \end{equation*}
We will use an interpolation technique to directly achieve the result stated in the Lemma. Let $\phi_i^t = \sqrt{t}\phi_i+\sqrt{1-t}\phi'_i$ be rvs that interpolates between $\phi'_i$ and $\phi_i$ as $t$ goes from $0\to1$. Define $S(t) := \bE{\xi,\xi',z}{f(\phi_1^t)f(\phi_2^t)}$. Note that $S(0) = \bE{\xi',z}{f(\phi'_1)f(\phi'_2)}$ and $S(1) = \bE{\xi}{f(\phi_1)f(\phi_2)}$. Thus, we need to show that $|\frac{\dd S(t)}{\dd t}|=O(1/d^{1-\epsilon})$ for all $t$.

Taking derivative of $S$, we have
    \begin{align*}
        \frac{dS}{dt} &=  \frac{1}{2}\underbrace{\bE{\xi,\xi',z}{f'(\phi^t_1)f(\phi_2^t)(\frac{1}{\sqrt{t}}\phi_1-\frac{1}{\sqrt{1-t}}\phi'_1)}}_{:=T_1}+\frac{1}{2}\underbrace{\bE{\xi,\xi',z}{f(\phi^t_1)f'(\phi_2^t)(\frac{1}{\sqrt{t}}\phi_2-\frac{1}{\sqrt{1-t}}\phi'_2)}}_{:=T_2},\\
    \end{align*}
It suffices to show that $T_1 = O(1/d^{1-\epsilon})$, as the conclusion for $T_2$ follows by symmetry. First, we recall an important property of cumulants:
let $h = (h_1,h_2)$ be random variables with joint cumulants $a_{(l_1, l_2)}$ for index $(l_1,l_2)$. Then we have the expansion
\begin{equation}\label{eqn:cumulant_expansion}
    \bE{}{h_1g(h)} = \sum_{k_1,k_2\ge 0}\frac{a_{(k_1+1,k_2)}}{k_1!k_2!}\bE{}{g^{(k_1, k_2)}(h)},
\end{equation}
where $g^{(l_1,l_2)}(h_1,h_2) = \frac{\partial^{l_1}\partial^{l_2}g(h_1,h_2)}{\partial h_1^{l_1}\partial h_1^{l_2}}$. This formula can be found, for example, in \cite{erdos_matrix_2019} pg. 30. 

Let $\kappa$ be the joint cumulants of $(\phi_1,\phi_2)$ and $\kappa'$ the joint cumulants of $(\phi'_1,\phi'_2)$. Using \eqref{eqn:cumulant_expansion}, we have
\begin{align*}
    &\bE{\xi,\xi',z}{f'(\phi_1^t)f(\phi_2^t)\phi_1} = \sum_{k_1,k_2\ge 0}\frac{\kappa_{(k_1+1,k_2)}}{k_1!k_2!}\bE{}{f^{(k_1+1)}(\phi_1^t)f^{(k_2)}(\phi_2^t)(\sqrt{t})^{k_1+k_2}}.
\end{align*}
Note that for polynomial $f$, the above expansion has a finite number of terms and there are no convergence issues.

Let $Q(\beta) = \log \bE{\xi}{e^{\beta^T\man{\rat{M}{D}\xi}}}$. 
The joint CGF of $(\phi_1,\phi_2)$ is given by 
$$
K(s_1,s_2) = \log \bE{\xi}{e^{\left(s_1\rat{w_1}{d}+s_2\rat{w_2}{d}\right)^T\man{\rat{M}{D}\xi}}} 
$$
The $(l_1, l_2)^{th}$ order joint cumulant of $(\phi_1, \phi_2)$ is given by 
$$
\kappa_{(l_1, l_2)}=\frac{\partial^{l_1}\partial^{l_2}}{\partial s_1^{l_1}\partial s_{2}^{l_2}} K(s_1,s_2)\bigg|_{s_1=s_2=0} = \sum_{i_1\cdots i_{l_1} j_1\cdots j_{l_2}}Q_{i_1\cdots i_{l_1}j_1\cdots j_{l_2}}^{(l_1+l_2)} \frac{w_1^{i_1}}{\sqrt d}\cdots \frac{w_1^{i_{l_1}}}{\sqrt d}
\frac{w_2^{j_1}}{\sqrt d}\cdots\frac{w_2^{j_{l_2}}}{\sqrt d}
$$
where $Q_{i_1\cdots i_{l_1}j_1\cdots j_{l_2}}^{l_1+l_2}$ is the symmetric tensor of order $l_1+l_2 = r$
\begin{align}
Q^{(r)}_{i_1\cdots i_r} & = \frac{\partial^r}{\partial\beta_{i_1}\cdots \partial\beta_{i_r}} \log\mathbb{E}[e^{\beta^T\sigma(\frac{M\xi}{\sqrt D})}]_{\beta =0},
\\ &
= {\rm Cum}\left[\man{\frac{(M\xi)_{i_1}}{\sqrt D}}, \man{\frac{(M\xi)_{i_2}}{\sqrt D}}, \dots, \man{\frac{(M\xi)_{i_r}}{\sqrt D}}\right].
\end{align}
Using Hermite expansions of the activation,  $Q^{(r)}_{i_1...i_r}$ is given by a finite sum (for an activation having a finite number of Hermite components) over $n_1, n_2, \dots, n_m$ non-negative integers, of terms of the form
$$
{\rm Cum}[H_{n_1}(X_{i_1}), H_{n_2}(X_{i_2}), \dots, H_{n_r}(X_{i_r})]
$$
where we have set $X_i = \frac{(M\xi)_{i}}{\sqrt D}$. The joint cumulants of Hermite polynomials can be represented as a sum over connected multigraphs (see Appendix \ref{appndxsec:cum_of_hermite}) and the dominant contribution to the sum easily identified. This allows to compute the mean and variance of $\kappa_{(l_1, l_2)}$. In Appendix~\ref{appndx_sec:cum_momemts} we derive the bound $\mathbb{E}[\kappa_{(l_1, l_2)}] = O(d^{-{1}})$, and $\mathbb{V}[\kappa_{(l_1, l_2)}] = O(D^{-2})$ valid for all $l_1+l_2\geq 3$. The precise combinatorial constant prefactors depend on $l_1, l_2$ and $\sigma$, but their explicit form is not needed. Note that better bounds with a better exponent typically depending on $l_1+l_2$ can be derived but this one is enough for our purposes here.

Then we want an estimate on 
$\mathbb{P}\bigl[\vert\kappa_{(l_1, l_2)} - \mathbb{E}[\kappa_{(l_1, l_2)}]\vert \geq \frac{1}{d^{1-\epsilon}}\bigr]$. Using Chebyshev's inequality gives an immediate estimate, however this is not strong enough for our purposes, and usual Chernoff bounds are difficult to handle because one then needs a control of the moment generating functions. In the present case however, $\kappa_{(l_1, l_2)}$ is a Gaussian multinomial and we can use hypercontractive estimates that require only control of the variance. For any polynomial $P(w)$ in $w\in \mathbb{R}^r$ a Gaussian vector with i.i.d. $\mathcal{N}(0, 1)$ components we have \citep[Theorem 5.10 in Chapter 5]{janson_gaussian_1997} 
$$
\Vert P(w)\Vert_p \leq (p-1)^{r/2} \Vert P(w)\Vert_2
$$
where $\Vert X\Vert_p = (\mathbb{E} \vert X\vert^p)^{1/p}$ for any $p\geq 2$.
Combining this bound with a moment Chernoff bound applied to $P(w) = \kappa_{(l_1, l_2)} - \mathbb{E}[\kappa_{(l_1, l_2)}]$, with $r=l_1+l_2\geq 3$, we get for 
any small $0<\epsilon \ll 1/2$:
\begin{align*}
\mathbb{P}\biggl[\vert\kappa_{(l_1, l_2)} - \mathbb{E}[\kappa_{(l_1, l_2)}]\vert \geq \frac{1}{d^{1-\epsilon}}\biggr] 
&
\leq d^{(1-\epsilon)p} \,
\mathbb{E}\biggl[\vert\kappa_{(l_1, l_2)} - \mathbb{E}[\kappa_{(l_1, l_2)}]\vert^p\biggr]
\\ & 
\leq 
\biggl(d^{1-\epsilon} (p-1)^{\frac{l_1+l_2}{2}} \sqrt{\mathbb{V}[\kappa_{(l_1, l_2)}]}\biggr)^p
\\&
\leq 
\biggl( \frac{C_{l_1, l_2, \sigma} (p-1)^{\frac{l_1+l_2}{2}}}{d^\epsilon}\biggr)^p
\end{align*}
In the last inequality we used the worst case bound $\mathbb{V}[\kappa_{(l_1, l_2)}] \leq  C_{l_1, l_2, \sigma} D^{-2}$ for $l_1+l_2\geq 3$.
Optimizing over $p\geq 2$ reveals that we can take an integer on the scale $p_* = \Theta(d^{\frac{2\epsilon}{l_1+l_2}})$ (up to constants depending on $l_1, l_2$) and after some computation we find 
$$
\mathbb{P}\biggl[\vert\kappa_{(l_1, l_2)} - \mathbb{E}[\kappa_{(l_1, l_2)}]\vert \geq \frac{1}{d^{1-\epsilon}}\biggr] 
\leq 
\exp\biggl({-c_{l_1, l_2, \sigma} \, d^{\frac{2\epsilon}{l_1+l_2}}}\biggr)
$$
for $0<\epsilon_0\leq \epsilon <<1$ and $d \geq d_0(l_1, l_2, \epsilon_0, \sigma)$ large enough. With the hypothesis in the Lemma on $f$ and $\sigma$, choosing $1\>> \epsilon>\epsilon_0>0$, the $d_0$ can be fixed to some large integer eventually depending only on $f, \sigma, \epsilon_0$.

In order to use this result in a convenient way we still need to make a few remarks. 
It is easy to see that the worst case bound $\mathbb{E}[\kappa_{(l_1, l_2)}] = O(d^{-{1}})$ holds for all $l_1+l_2\geq 3$ (in fact for $l_1+l_2=3$ the expectation is zero by symmetry of $w$'s, the case $l_1+l_2 =4$ gives exactly $O(d^{-1})$, while higher cumulants give better exponent. Thus, we have
$$
\mathbb{P}\biggl[\vert\kappa_{(l_1, l_2)}\vert \geq \bigO{\frac{1}{d^{1-\epsilon}}}\biggr] \leq \exp\biggl({-c_{l_1, l_2, \sigma} \, d^{\frac{2\epsilon}{l_1+l_2}}}\biggr).
$$
Now, in fact we want to estimate the probability of events of the form: 
\begin{align}\label{eq:event}
\vert\kappa_{(l_1, l_2)}\vert \vert F(w_1, w_2)\vert \geq \frac{1}{d^{1-\epsilon}}
\end{align}
for some functions $F(w_1, w_2)$. Fix $\eta$ such that $0<\eta \le \epsilon/2 \ll 1/2$. Notice that event \eqref{eq:event} implies $\vert\kappa_{(l_1, l_2)}\vert \geq  \frac{1}{d^{1-\eta}}$ or $\vert F(w_1, w_2)\vert \geq d^{\eta}$ for $0<\eta \le \epsilon/2 \ll 1/2$ (this is seen by looking at the contrapositive of this assertion).
Moreover for polynomial $F(w_1, w_2)$ and Gaussian $w$'s it is clear that the event $\vert F(w_1, w_2)\vert \geq d^{\eta}$ has low probability tending to zero as $\bigO{e^{-c\, d^{2\eta}}}$. 
Putting these facts together and using the union bound with $\eta=\epsilon/2$, we conclude that the probability of the event \eqref{eq:event} is at most $O\biggl(\exp({-c_{l_1, l_2, \sigma} \, d^{\frac{2\epsilon}{l_1+l_2}}})\biggr)$, and it's complement has probability at least $1-O\biggl(\exp({-c_{l_1, l_2, \sigma} \, d^{\frac{2\epsilon}{l_1+l_2}}})\biggr)$.

From now on it is understood that the following estimates are valid w.h.p $1-O\biggl(\exp({-c_{f, \sigma} \, d^{\frac{2\epsilon}{l_1+l_2}}})\biggr)$ for some constant 
$c_{f, \sigma} >0$.
By using \eqref{eqn:cumulant_expansion} for $T_1$, we have
\begin{align}\label{apndx_eqn:interpolation_t1}
    &\bE{\xi,\xi',z}{f'(\phi_1^t)f(\phi_2^t)\phi_1} = \sum_{k_1,k_2\ge 0}\frac{\kappa_{(k_1+1,k_2)}}{k_1!k_2!}\bE{}{f^{(k_1+1)}(\phi_1^t)f^{(k_2)}(\phi_2^t)(\sqrt{t})^{k_1+k_2}},\nonumber\\
    &\stackrel{(a)}{=}\kappa_{(1,0)}\bE{}{f^{(1)}(\phi_1^t)f(\phi_2^t)}+\sqrt{t}\kappa_{(1,1)}\bE{}{f^{(1)}(\phi_1^t)f^{(1)}(\phi_2^t)}+\frac{\sqrt{t}}{2}\kappa_{(2,0)}\bE{}{f^{(2)}(\phi_1^t)f(\phi_2^t)}+O(1/d^{1-\epsilon}),\nonumber\\
    &\stackrel{(b)}{=} \sqrt{t}\kappa_{(1,1)}\bE{}{f^{(1)}(\phi_1^t)f^{(1)}(\phi_2^t)}+\frac{\sqrt{t}}{2}\kappa_{(2,0)}\bE{}{f^{(2)}(\phi_1^t)f(\phi_2^t)}+O(1/d^{1-\epsilon}),
\end{align}
where in $(a)$ we kept only the cumulants with $k_1+k_2\le 2$ owing to the previous discussion, and $(b)$ follows from the fact that $\kappa_{(1,0)}\bE{}{f^{(1)}(\phi_1^t)f(\phi_2^t)} = O(1/d)$ since $\kappa_{(1,0)} = O(1/\sqrt{d})$ (due to CLT) and $\bE{}{f^{(1)}(\phi_1^t)f(\phi_2^t)} = O(1/\sqrt{d})$ due to the zero-mean of $f$ (first use CLT and then expand using Mehler's kernel formula to see this).
Similarly, we have
\begin{align}\label{apndx_eqn:interpolation_t2}
    &\bE{\xi,\xi',z}{f'(\phi_1^t)f(\phi_2^t)\phi_1'} = \sum_{k_1,k_2\ge 0}\frac{\kappa'_{(k_1+1,k_2)}}{k_1!k_2!}\bE{}{f^{(k_1+1)}(\phi_1^t)f^{(k_2)}(\phi_2^t)(\sqrt{1-t})^{k_1+k_2}},\nonumber\\
    &=\sqrt{1-t}\kappa'_{(1,1)}\bE{}{f^{(1)}(\phi_1^t)f^{(1)}(\phi_2^t)}+\frac{\sqrt{1-t}}{2}\kappa'_{(2,0)}\bE{}{f^{(2)}(\phi_1^t)f(\phi_2^t)},\\
\end{align}
where we used the fact that cumulant for Gaussian are zero for $k_1+k_2\ge 3$, and $\kappa'_{(1,0)} = 0$ due to $\bE{}{\phi_1'}=0$. 
Substituting \eqref{apndx_eqn:interpolation_t1} and \eqref{apndx_eqn:interpolation_t2} in the expression for $T_1$, we have
\begin{align*}
    T_1 &= (\kappa_{(1,1)}-\kappa'_{(1,1)})\bE{}{f^{(1)}(\phi_1^t)f^{(1)}(\phi_2^t)}+\frac{1}{2}(\kappa_{(2,0)}-\kappa'_{(2,0)})\bE{}{f^{(2)}(\phi_1^t)f(\phi_2^t)}+O(1/d^{1-\epsilon/2}).
\end{align*}
We have $\kappa_{(1,1)} = \rat{w_1^T}{d}Q^{(2)}\rat{w_2}{d}$, and $\kappa_{(1,1)}' = \rat{w_1^T}{d}Q'^{(2)}\rat{w_2}{d}$. Using Mehler kernel formula, we get the following for $Q^{(2)}$
\begin{align*}
    Q^{(2)} &= \bE{\xi}{\man{\rat{M}{D}\xi}\man{\rat{M}{D}\xi}^T},\\
    &= \nu_1^2\frac{MM^T}{D}+(1-\nu_1^2)I + \bigO{1/d},\\
    &= Q'^{(2)}+\bigO{1/d}.
\end{align*}
Thus, we have $\bE{w_1,w_2}{\kappa_{(1,1)}-\kappa'_{(1,1)}} = 0$, and $\mathbb{V}[\kappa_{(1,1)}-\kappa'_{(1,1)}] = \frac{1}{d^2}\norm{Q^{(2)}-Q'^{(2)}}_F^2 = \bigO{1/d^2}$. Thus, by Chebyshev inequality $\vert\kappa_{(1,1)}-\kappa'_{(1,1)}\vert = \bigO{1/d^{1-\epsilon}}$ w.h.p.
Moreover, CLT gives $(\kappa_{(2,0)}-\kappa'_{(2,0)}) = O(1/\sqrt{d})$ while the zero mean of $f$ means $\bE{}{f^{(2)}(\phi_1^t)f(\phi_2^t)} = O(1/\sqrt{d})$ making their product $O(1/d)$. Hence, we get
\begin{align*}
    T_1 &= O(1/d^{1-\epsilon}).
\end{align*}
\end{proof}
\subsection{Expectations and cumulants for products of Hermite polynomials}\label{appndxsec:cum_of_hermite}

Let $X_1,\dots, X_m$ be scalar Gaussian random variables with mean zero and covariance matrix $r_{ij} = \mathbb{E}[X_iX_j]$. Without loss of generality we assume $r_{ii} =1$ (otherwise in what follows one must replace $r_ij$ by $r_{ij}/\sqrt{r_{ii} r_{jj}})$.
Note that in our application we have $X_i = \frac{(M\xi)_i}{\sqrt D}$ with $\xi\sim\mathcal{N}(0, I_d)$ and $r_{ij} = \frac{1}{D}M_i^T M_j$ where $M_k$ is the $k-th$ row of $M\in \mathbb{R}^{d\times D}$. Therefore for large enough dimensions $r_{ii} \approx 1$.

Consider the set $\Gamma_m\owns\gamma$ of graphs with $m$ vertices $\{1,\dots,m\}$ with degrees $n_1,\dots, n_m$ and edges connecting pairs of distinct vertices $i\neq j$ (no self-loops). Pairs of nodes can have a multi-edges. If we denote by $k_{ij}$ ($=k_{ji}$) the number of edges connecting the pair $i, j$, we have the constraint $\sum_{j\neq i} k_{ij} = n_i$ for each $i$. The set of fully connected such graphs is denoted $\Gamma_m^c\owns\gamma^c$. 
Note that there is a one-to-one correspondence between a graph $\gamma\in \Gamma_m$ (or $\gamma^c\in \Gamma_m^c$) completely determines the degrees $n_1, ..., n_m$ and the number of multi-edges $k_{ij}$ between pairs $i\neq j$.

\begin{lemma}\label{lem:av_and_cum_hermite}
The average of Hermite products is given by the sum over all graphs:
\begin{align}\label{eq:avproduct}
\mathbb{E}[H_{n_1}(X_1)H_{n_2}(X_2)...H_{n_m}(X_m)] = \sum_{\gamma\in \Gamma_m}
\frac{\prod_{i=1}^m n_i!}{\prod_{1\le i<j\le m} k_{ij}!}
\prod_{1\le i<j\le m} r_{ij}^{\,k_{ij}},
\end{align}
In the cumulant of Hermite polynomials only fully connected graphs contribute:
\begin{align}\label{eq:avcum}
{\rm Cum}[H_{n_1}(X_1), H_{n_2}(X_2), ..., H_{n_m}(X_m)] = \sum_{\gamma^c\in \Gamma_m^c}
\frac{\prod_{i=1}^m n_i!}{\prod_{1\le i<j\le m} k_{ij}!}
\prod_{1\le i<j\le m} r_{ij}^{\,k_{ij}},
\end{align}
\end{lemma}
This lemma is a generalization of the usual Isserli or Wick theorem for products of Gaussian variables (corresponding to $n_1=n_2=\dots n_m=1$). It is a form of the Kibble-Slepian formula \cite{kibble_extension_1945,slepian_symmetrized_1972}, but also of Wick's theorem for bosons in quantum mechanics \cite{chou_equilibrium_1985}. For the convenience of the reader we give here a streamlined proof.

\begin{proof}
We first prove \eqref{eq:avproduct} by the generating function method. The Hermite polynomials here are defined from the following generating function
\[
e^{xt-t^2/2}=\sum_{n=0}^\infty \mathrm{H}_n(x)\frac{t^n}{n!}.
\]
Therefore we introduce 
\[
G(t_1,\dots,t_m)
:=
\mathbb E\!\left[\prod_{i=1}^m e^{t_iX_i-t_i^2/2}\right].
\]
for which we have
\[
G(t_1,\dots,t_m)
=
\sum_{n_1,\dots,n_m\ge 0}
\mathbb E\!\left[\prod_{i=1}^m \mathrm{H}_{n_i}(X_i)\right]
\prod_{i=1}^m \frac{t_i^{n_i}}{n_i!}.
\tag{1}
\]
Since $X_i$ are mean zero with covariance
\(R=(r_{ij})\) we have
\[
G(t_1,\dots,t_m)
=
\mathbb E\!\left[e^{\sum_{i=1}^m t_iX_i-\frac12\sum_{i=1}^m t_i^2}\right]
=
\exp\!\left(\frac12\, t^\top R t-\frac12\sum_{i=1}^m t_i^2\right).
\]
With \(r_{ii}=1\), we obtain
\[
\frac12\, t^\top R t-\frac12\sum_{i=1}^m t_i^2
=
\sum_{1\le i<j\le m} r_{ij}t_it_j,
\]
consequently
\[
G(t_1,\dots,t_m)
=
\exp\!\left(\sum_{1\le i<j\le m} r_{ij}t_it_j\right).
\tag{2}
\]
Expanding the r.h.s
\begin{align*}
G(t_1,\dots,t_m)
& =
\prod_{1\le i<j\le m}\exp(r_{ij}t_it_j)
=
\prod_{1\le i<j\le m}
\left(
\sum_{k_{ij}=0}^\infty \frac{(r_{ij}t_it_j)^{k_{ij}}}{k_{ij}!}
\right),
\\ &
=
\sum_{(k_{ij})_{i<j}}
\left(
\prod_{1\le i<j\le m}\frac{r_{ij}^{k_{ij}}}{k_{ij}!}
\right)
\prod_{1\le i<j\le m}(t_it_j)^{k_{ij}}
\\ &
=
\sum_{(k_{ij})_{i<j}}
\left(
\prod_{1\le i<j\le m}\frac{r_{ij}^{k_{ij}}}{k_{ij}!}
\right)
\prod_{i=1}^m t_i^{\sum_{j\ne i}k_{ij}}.
\tag{3}
\end{align*}
The coefficient of \(t_1^{n_1}\cdots t_m^{n_m}\) in \((3)\) is
\[
\sum_{\substack{k_{ij}\ge 0\\ \sum_{j\ne i}k_{ij}=n_i}}
\prod_{1\le i<j\le m}\frac{r_{ij}^{k_{ij}}}{k_{ij}!}.
\]
But because of \((1)\), this coefficient is also equal to
\[
\frac{1}{n_1!\cdots n_m!}
\mathbb E\!\left[\prod_{i=1}^m \mathrm{H}_{n_i}(X_i)\right].
\]
We therefore conclude
\[
\mathbb E\!\left[\prod_{i=1}^m \mathrm{H}_{n_i}(X_i)\right]
=
\sum_{\substack{k_{ij}\ge 0\\ \sum_{j\ne i}k_{ij}=n_i}}
\frac{\prod_{i=1}^m n_i!}{\prod_{1\le i<j\le m} k_{ij}!}
\prod_{1\le i<j\le m} r_{ij}^{\,k_{ij}},
\]
which is precisely \eqref{eq:avproduct}.

Now we deduce \eqref{eq:avcum} by a standard argument. Note that any graph $\gamma\in \Gamma_m$ is a union of disjoint fully connected parts $\gamma = \cup_{l\in \rho} \gamma_l^c$ corresponding to a partition $\pi$ of the set of vertices $\{1,\dots,m\}$. The 'finest' partition is $\{\{1\}\{2\}\dots \{m\}\}$ and the 'grossest' one is $\{1\dots m\}$. We have the partial order $\rho\preceq \pi$ ($\rho$ is finer than $\pi$) if all subsets of $\rho$ are included in the subsets of $\pi$. 

Cumulants are related to moments through:
\begin{equation}\label{eq:momentcum}
{\rm Cum}[H_{n_1}(X_1),
\dots,H_{n_m}(X_m)]
=
\sum_{\pi}
(|\pi|-1)!(-1)^{|\pi|-1}
\prod_{B\in\pi}\shortexpect\left[\prod_{i\in B}H_{n_i}(X_i)\right].
\end{equation}
For a block \(B\in \pi\), the expectation
\[
\shortexpect\left[\prod_{i\in B}H_{n_i}(X_i)\right]
\]
is itself a sum over loopless graphs whose vertices are the elements of \(B\), with degrees \(n_i\). Hence, for a partition \(\pi\), the product over blocks
\[
\prod_{B\in\pi}\shortexpect\left[\prod_{i\in B}H_{n_i}(X_i)\right]
\]
is a sum over graphs on \(\{1,
\dots,m\}\) whose multi-edges lie entirely inside the blocks of \(\pi\). In other words, no edge is allowed to connect two different blocks of \(\pi\).
Thus a fixed graph \(\gamma\) contributes to the term indexed by \(\pi\) in \eqref{eq:momentcum} exactly when every connected component of \(\gamma\) is contained in a block of \(\pi\). Equivalently, if \(\Pi(\gamma)\) denotes the partition of \(\{1,
\dots,m\}\) into connected components of \(\gamma\), then \(\gamma\) appears precisely for partitions \(\pi\) satisfying
$\Pi(\gamma)\preceq \pi$.
Therefore the total coefficient multiplying a fixed graph \(\Gamma\) in the cumulant is
\begin{equation}\label{eq:MobiusSum}
\sum_{\pi:\ \Pi(\gamma)\preceq \pi}
(|\pi|-1)!(-1)^{|\pi|-1}.
\end{equation}
By the Moebius Lemma \ref{lem:moebius} below this sum is equal to \(1\) if \(\Pi(\gamma)\) has only one block namely $\{1\dots m\}$, and equal to \(0\) otherwise. This means that only fully connected graphs contribute to the cumulant.

\end{proof}

\begin{lemma}\label{lem:moebius}
Let $\mathcal P_m$ be the set of partitions of $\{1,\dots,m\}$.
Recall that for partitions $\rho,\pi$ of $\{1\dots m\}$, we write $\rho\preceq \pi$ if $\rho$ is finer than $\pi$,
that is, every block of $\rho$ is contained in a block of $\pi$.
Then
\[
\sum_{\pi:\,\rho\preceq \pi}
(|\pi|-1)!(-1)^{|\pi|-1}
=
\begin{cases}
1,& \text{if } \rho=\{\{1,\dots,m\}\},\\
0,& \text{otherwise.}
\end{cases}
\]
\end{lemma}
\begin{proof}
Let $k=|\rho|$ be the number of blocks of $\rho$. Any partition
$\pi$ such that $\rho\preceq \pi$ is obtained by merging the $k$ blocks
of $\rho$. Therefore such $\pi$ are in bijection with partitions of a
$k$-element set.
Thus the sum in teh lemma depends only on $k$ and equals
\[
S_k
=
\sum_{j=1}^k
\left\{ {k \atop j} \right\}
(j-1)!(-1)^{j-1}.
\]
where $\left\{ {k \atop j} \right\}$ is teh Stirling number (the number of partitions of a $k$ element set into $j$ blocks).
Using the exponential generating function of Stirling's numbers:
\[
\sum_{k\ge j}
\left\{ {k \atop j} \right\}
\frac{x^k}{k!}
=
\frac{(e^x-1)^j}{j!},
\]
we obtain
\[
\sum_{k\ge 1} S_k\frac{x^k}{k!}
=
\sum_{j\ge 1}
\frac{(-1)^{j-1}}{j}
(e^x-1)^j
=
\log(e^x)=x.
\]
Hence $S_1=1$ and $S_k=0$ for $k\ge2$.
\end{proof}

\subsection{Mean and variance of cumulants $\kappa_{(l_1, l_2)}$}\label{appndx_sec:cum_momemts}
In this section, we derive bounds for the mean and variance of $\kappa_{(l_1,l_2)}$, the joint cumulants of $\rat{w_1^T}{d}\man{\rat{M}{D}\xi}$, and $\rat{w_2^T}{d}\man{\rat{M}{D}\xi}$, for $l_1\ge 0, l_2\ge 0$, and $r=l_1+l_2\ge 3$. Recall that $\kappa_{(l_1, l_2)}$ is given by 
$$
\kappa_{(l_1, l_2)} = \sum_{i_1\cdots i_{l_1} j_1\cdots j_{l_2}}Q_{i_1\cdots i_{l_1}j_1\cdots j_{l_2}}^{(r)} \frac{w_1^{i_1}}{\sqrt d}\cdots\frac{w_1^{i_{l_1}}}{\sqrt d}
\frac{w_2^{j_1}}{\sqrt d}\cdots\frac{w_2^{j_{l_2}}}{\sqrt d}
$$
where $Q_{i_1\cdots i_{l_1}j_1\cdots j_{l_2}}^{(r)}$ is the symmetric tensor of order $l_1+l_2 = r$ given by
\begin{align}
Q^{(r)}_{i_1\cdots i_r} & = \frac{\partial^r}{\partial\beta_{i_1}\cdots \partial\beta_{i_r}} \log\mathbb{E}[e^{\beta^T\sigma(\frac{M\xi}{\sqrt D})}]_{\beta =0},
\\ &
= {\rm Cum}\left[\man{\frac{(M\xi)_{i_1}}{\sqrt D}}, \man{\frac{(M\xi)_{i_2}}{\sqrt D}}, \dots, \man{\frac{(M\xi)_{i_r}}{\sqrt D}}\right].
\end{align}
\subsubsection{Mean}
Consider the first moment of $\kappa_{(l_1,l_2)}$. It is zero when $l_1$ or $l_2$ is odd. Otherwise, we have
\begin{align*}
    \bE{}{\kappa_{(l_1,l_2)}} &= \frac{1}{d^{r/2}}\sum_{i_1\cdots i_{l_1} j_1\cdots j_{l_2}}Q_{i_1\cdots i_{l_1}j_1\cdots j_{l_2}}^{(r)} \bE{}{w_1^{i_1}\cdots w_1^{i_{l_1}}}
\bE{}{w_2^{j_1}\cdots w_2^{j_{l_2}}}.
\end{align*}
Using Wick's theorem, the expectations can be written as the sum over pairings of indices. Let $\cK_{l}$ be the set of all pairings of $(1,\cdots,l)$. Then,
\begin{align*}
    \bE{}{\kappa_{(l_1,l_2)}} &= \frac{1}{d^{r/2}}\sum_{i_1\cdots i_{l_1} j_1\cdots j_{l_2}}Q_{i_1\cdots i_{l_1}j_1\cdots j_{l_2}}^{(r)} \left(\sum_{\pi_1\in\cK_{l_1}}\prod_{(a_1,b_1)\in\pi_1}\bE{}{w_1^{i_{a_1}}w_1^{i_{b_1}}}\right)
\left(\sum_{\pi_2\in\cK_{l_2}}\prod_{(a_2,b_2)\in\pi_2}\bE{}{w_2^{j_{a_2}}w_2^{j_{b_2}}}\right),\\
&= \frac{1}{d^{r/2}}\sum_{i_1\cdots i_{l_1} j_1\cdots j_{l_2}}Q_{i_1\cdots i_{l_1}j_1\cdots j_{l_2}}^{(r)} \sum_{\substack{\pi_1\in\cK_{l_1}\\\pi_2\in\cK_{l_2}}}\prod_{(a_1,b_1)\in\pi_1}\delta_{i_{a_1}i_{b_1}}\prod_{(a_2,b_2)\in\pi_2}\delta_{j_{a_2}j_{b_2}},\\
&= \frac{1}{d^{r/2}}\sum_{\substack{\pi_1\in\cK_{l_1}\\\pi_2\in\cK_{l_2}}}\sum_{i_1\cdots i_{l_1} j_1\cdots j_{l_2}}Q_{i_1\cdots i_{l_1}j_1\cdots j_{l_2}}^{(r)} \prod_{(a_1,b_1)\in\pi_1}\delta_{i_{a_1}i_{b_1}}\prod_{(a_2,b_2)\in\pi_2}\delta_{j_{a_2}j_{b_2}}.
\end{align*}
Each pairing in the above equation leads to same value due to the symmetric nature of $Q^{(r)}$. Thus,
\begin{align*}
    \bE{}{\kappa_{(l_1,l_2)}} &=  C_r\frac{1}{d^{r/2}}\sum_{i_1\cdots i_{l_1/2} j_1\cdots j_{l_2/2}}Q_{i_1i_1i_2i_2\cdots i_{l_1/2}i_{l_1/2}j_1j_1\cdots j_{l_2/2}j_{l_2/2}}^{(r)},
\end{align*}
where $C_r = \frac{l_1!l_2!}{2^{r/2}(l_1/2)!(l_2/2)!}$. We have
\begin{align}
Q^{(r)}_{i_1\cdots i_r} &= {\rm Cum}\left[\man{\frac{(M\xi)_{i_1}}{\sqrt D}}, \man{\frac{(M\xi)_{i_2}}{\sqrt D}}, \dots, \man{\frac{(M\xi)_{i_r}}{\sqrt D}}\right],\\
&= \sum_{n_1,n_2,\cdots,n_r}\nu_{n_1}\nu_{n_2}\cdots\nu_{n_r} {\rm Cum}[H_{n_1}(X_{i_1}), H_{n_2}(X_{i_2}), \dots, H_{n_r}(X_{i_r})]
\end{align}
where we have set $X_i = \frac{(M\xi)_{i}}{\sqrt D}$.

From the diagrammatic analysis in Appendix~\ref{appndxsec:cum_of_hermite}, the leading scaling in $d$ is determined by the minimum number of edges required to connect the distinct indices appearing among $(i_1,\cdots,i_r)$. If there are $m$ distinct indices, the minimal connected graph has $m-1$ edges, leading to scaling
\begin{equation*}
    Q^{(r)}_{i_1\cdots i_r} = \bigO{d^{-(m-1)/2}}.
\end{equation*}

For $Q_{i_1i_1i_2i_2\cdots i_{l_1/2}i_{l_1/2}j_1j_1\cdots j_{l_2/2}j_{l_2/2}}^{(r)}$, since there are only $r/2$ distinct indices at most, which leads to
\begin{align*}
    \bE{}{\kappa_{(l_1,l_2)}} & = C_r\frac{1}{d^{r/2}}\sum_{i_1\cdots i_{l_1/2} j_1\cdots j_{l_2/2}}Q_{i_1i_1i_2i_2\cdots i_{l_1/2}i_{l_1/2}j_1j_1\cdots j_{l_2/2}j_{l_2/2}}^{(r)}\\
& = C_r\frac{1}{d^{r/2}} d^{r/2}\bigO{{d}^{-(r/2-1)/2}} = \bigO{{d}^{-(r/2-1)/2}}.
\end{align*}
Note that in the above expression, we considered only the terms with distinct indices since they make the leading order contribution.
Thus, for $r\ge 6$, $\bE{}{\kappa_{(l_1,l_2)}} = \bigO{\frac{1}{d}}$.
When $l_1=l_2 =2$, 
\begin{align*}
    \bE{}{\kappa_{(2,2)}} &= \frac{1}{d^2}\sum_{i_1i_2j_1j_2}Q_{i_1i_2j_1j_2}^{(r)} \bE{}{w_1^{i_1}w_1^{i_2}}
\bE{}{w_2^{j_1}w_2^{j_2}}\\
& = \frac{1}{d^2}\sum_{i_1j_1}Q_{i_1i_1j_1j_1}^{(r)},\\
&= \frac{1}{d^2}\sum_{i_1j_1}\frac{M_{i_1}^TM_{j_1}}{D} + \bigO{\frac{1}{d}},\\
&= \frac{1}{d^2}\sum_{i_1\neq j_1}\frac{M_{i_1}^TM_{j_1}}{D} + \frac{1}{d^2}\sum_{i_1}\frac{M_{i_1}^TM_{i_1}}{D} + \bigO{\frac{1}{d}},\\
&= \frac{1}{d^2}\sum_{i_1}\frac{M_{i_1}^T\sum_{j_1\neq i_1}M_{j_1}}{D} + \bigO{\frac{1}{d}},\\
&= \bigO{\frac{1}{d}}.
\end{align*}
A similar argument shows that $\bE{}{\kappa_{4,0}} = \bE{}{\kappa_{0,4}} = \bigO{\frac{1}{d}}$. Hence, for any $r=l_1+l_2\ge 3$, we have $\bE{}{\kappa_{(l_1,l_2)}} = \bigO{\frac{1}{d}}$.

\subsubsection{Variance}
Next, we consider variance of $\kappa_{(l_1,l_2)}$ for $r=l_1+l_2\ge 3$. We can expand to get
\begin{align*}
    \mathbb{V}[\kappa_{(l_1,l_2)}] &= \frac{1}{d^r}\sum_{\substack{i_1\cdots i_{l_1} j_1\cdots j_{l_2}\\i_1'\cdots i_{l_1}' j_1'\cdots j_{l_2}'}}Q_{i_1\cdots i_{l_1}j_1\cdots j_{l_2}}^{(r)} Q_{i_1'\cdots i_{l_1}'j_1'\cdots j_{l_2}'}^{(r)}\text{Cov}\left(w_1^{i_1}\cdots w_1^{i_{l_1}}
w_2^{j_1}\cdots w_2^{j_{l_2}},w_1^{i_1'}\cdots w_1^{i_{l_1}'}
w_2^{j_1'}\cdots w_2^{j_{l_2}'}\right)
\end{align*}
Applying Wick's theorem to the covariance, we see that the contributions to the above summation will depend on the the number of internal pairings within two sets of vertices: $(i_1,\cdots,i_{l_1},j_1,\cdots,j_{l_2})$ and $(i_1',\cdots,i_{l_1}',j_1',\cdots,j_{l_2}')$. Let $\cK_{s}$ denote the set of all pairings of indices $(i_1,\cdots i_{l_1},j_1,\cdots j_{l_2},i_1',\cdots i_{l_1}',j_1',\cdots j_{l_2}')$ with 1) $s$ pairings internally within $(i_1,\cdots i_{l_1},j_1,\cdots j_{l_2})$ (and $(i_1',\cdots i_{l_1}',j_1',\cdots j_{l_2}')$), 2) no pairings between $(i_1,\cdots i_{l_1})$ and $(j_1,\cdots j_{l_2})$, and also no pairings between $(i_1',\cdots i_{l_1}')$ and $(j_1',\cdots j_{l_2}')$, 3) at least one pairing between $(i_1,\cdots i_{l_1},j_1,\cdots j_{l_2})$ and $(i_1',\cdots i_{l_1}',j_1',\cdots j_{l_2}')$. Then,
\begin{align*}
    \mathbb{V}[\kappa_{(l_1,l_2)}] &= \frac{1}{d^r}\sum_{\substack{i_1\cdots i_{l_1} j_1\cdots j_{l_2}\\i_1'\cdots i_{l_1}' j_1'\cdots j_{l_2}'}}Q_{i_1\cdots i_{l_1}j_1\cdots j_{l_2}}^{(r)} Q_{i_1'\cdots i_{l_1}'j_1'\cdots j_{l_2}'}^{(r)}\sum_{s=0}^{\lfloor\frac{r-1}{2}\rfloor} \sum_{\pi\in\cK_s} \prod_{(a,b)\in\pi}\delta_{ab},\\
    &= \sum_{s=0}^{\lfloor\frac{r-1}{2}\rfloor} \mathbb{V}[\kappa_{(l_1,l_2)}]_s,
\end{align*}
where $\mathbb{V}[\kappa_{(l_1,l_2)}]_s$ denote the contribution to the variance due to graphs with $s$ internal pairings. Let $C_{r,s}'$ denote the number of such graphs. Then,
\begin{align*}
    \mathbb{V}[\kappa_{(l_1,l_2)}]_s &= C_{r,s}'\frac{1}{d^r}\sum_{\substack{i_1\cdots i_{l_1} j_1\cdots j_{l_2}\\i_1'\cdots i_{l_1}' j_1'\cdots j_{l_2}'\\ s \text{ internal pairings}} }Q_{i_1\cdots i_{l_1}j_1\cdots j_{l_2}}^{(r)} Q_{i_1'\cdots i_{l_1}'j_1'\cdots j_{l_2}'}^{(r)},\\
    &=C_{r,s}'\frac{1}{d^r} \sum_{r-2s \text{ variables}}\left(\sum_{s \text{ variables}}Q_{i_1\cdots i_{l_1}j_1\cdots j_{l_2}}^{(r)}\right)^2,\\
    &\le \frac{1}{d^r} \sum_{r-s \text{ variables}}d^s {Q_{i_1\cdots i_{l_1}j_1\cdots j_{l_2}}^{(r)}}^2,\\
    &= \frac{1}{d^r}d^{r-s}d^s\bigO{\frac{1}{d^{r-s-1}}} = \bigO{\frac{1}{d^{r-s-1}}},
\end{align*}
where we used the fact that leading order in $Q^{(r)}$ with $s$ indices same is determined by the minimum number of edges in a connected graph with $r-s$ vertices, which is  $r-s-1$.
Note that $s$ is at most $r/2-1$ for even $r$ and $(r-1)/2$ for odd $r$. Thus, for $r\ge 4$, variance is $\bigO{\frac{1}{d^2}}$. Even in the case when $r=3$, and there are no internal pairings, the variance is $\bigO{\frac{1}{d^2}}$. When $r=3$ and $s=1$, we have
\begin{align*}
    \mathbb{V}[\kappa_{(l_1,l_2)}]_1 &= \frac{1}{d^3}\sum_{j}\left(\sum_{i}Q_{iij}\right)^2,\\
    &= \frac{1}{d^3}\sum_{j}\left(\sum_{i}\left(\frac{M_i^TM_j}{D}+\bigO{\frac{1}{d}}\right)\right)^2,\\
    &= \frac{1}{d^3}\sum_{j}\left(\frac{(\sum_{i\neq j}M_i)^TM_j}{D}+\bigO{1}\right)^2,\\
    &= \frac{1}{d^3}\sum_{j}\bigO{1},\\
    &= \bigO{\frac{1}{d^2}}.
\end{align*}
Hence, we have shown that for any $l_1+l_2=r\ge 3$, $\bE{}{\kappa_{(l_1,l_2)}} = \bigO{\frac{1}{d}}$ and $\mathbb{V}[\kappa_{(l_1,l_2)}] = \bigO{\frac{1}{d^2}}$.

\subsection{Proof of Lemma~\ref{lemma:norm_score}}\label{apndx_sec:proof_lemma_norm_score}
\begin{lemma*}
The test error for the exact score function can be expressed as
\begin{equation}
        \cE_{\text{test}}^\ast = \frac{1}{d}\shortexpect_{x,z}{\norm{\sqrt{h_t}\nabla\log P_t(a_tx+\sqrt{h_t}z)+z}^2} = 1-\frac{h_t}{d}\shortexpect_{y\sim P_t}{\norm{\nabla\log P_t(y)}^2}.
    \end{equation}
    Moreover,
    \begin{equation}
        \frac{1}{d}\shortexpect_{y\sim P_t}{\norm{\nabla\log P_t(y)}^2} = \frac{1}{h_t}\left(1-\frac{a_t^2}{h_t}\frac{1}{d}\text{MMSE}(x|y)\right),
    \end{equation}
    where $\text{MMSE}(x|y) = \bE{X\sim P_0,\; Y = a_tX+\sqrt{h_t}Z}{\norm{X-\bE{}{X|Y}}^2}$.
\end{lemma*}
\begin{proof}
The first statement follows by expanding the l.h.s.
    \begin{gather*}
        \frac{1}{d}\shortexpect_{x,z}{\norm{\sqrt{h_t}\nabla\log P_t(a_tx+\sqrt{h_t}z)+z}^2} = \frac{h_t}{d}\shortexpect_{y}{\norm{\nabla\log P_t(y)}^2}+\frac{2\sqrt{h_t}}{d}\bE{x,z}{z^T\nabla\log P_t(y)}+1,\\ 
        = \frac{h_t}{d}\shortexpect_{y}{\norm{\nabla\log P_t(y)}^2}+\frac{2\sqrt{h_t}}{d}\bE{x,z}{\expectCond{z}{y}^T\nabla\log P_t(y)}+1 = 1-\frac{h_t}{d}\shortexpect_{y}{\norm{\nabla\log P_t(y)}^2}.
    \end{gather*}
    Using $\nabla\log P_t(y) = \frac{a\expectCond{x}{a_tx+\sqrt{h_t}z=y}-y}{h_t}$, we have
    \begin{gather*}
        \frac{1}{d}\shortexpect_{y\sim P_t}{\norm{\nabla\log P_t(y)}^2} = \frac{1}{dh_t^2}\shortexpect_{y\sim P_t}{\norm{a_t\expectCond{x}{y}-y}^2},\\ = \frac{1}{dh_t^2}\bE{y\sim P_t}{a_t^2\norm{\expectCond{x}{y}}^2-2a_t\expectCond{y^Tx}{y}+\norm{y}^2}
        = \frac{1}{dh_t^2}\left\{a_t^2\shortexpect{\norm{\expectCond{x}{y}}^2}-2a_t\shortexpect{y^Tx}+\shortexpect{\norm{y}^2}\right\},\\ = \frac{1}{dh_t^2}\left\{a_t^2\left(\shortexpect{\norm{\expectCond{x}{y}}^2}-\shortexpect{\norm{x}^2}\right)+h_td\right\}
        = \frac{1}{h_t}\left(1-\frac{a_t^2}{h_t}\frac{1}{d}\text{MMSE}(x|y)\right).
    \end{gather*}
\end{proof}

\end{document}